\documentclass[11pt, a4paper]{article}

\usepackage[tablesfirst,notablist,nomarkers]{endfloat}
\usepackage{amssymb}
\usepackage{enumitem}
\usepackage{tabularx}
\usepackage{float}
\usepackage{multirow}
\usepackage{bbm}
\usepackage{graphicx}
\usepackage{mathrsfs}
\usepackage{subcaption}
\usepackage{svg}
\usepackage{setspace}
\usepackage{amsmath}
\usepackage{placeins}
\usepackage{booktabs}
\usepackage[
a4paper,top=3cm,bottom=2cm,left=3cm,right=3cm,marginparwidth=1.75cm]{geometry}
  
\usepackage{amsthm,bm,mathtools,etoolbox,amsfonts,bbm}
\usepackage[ruled,linesnumbered]{algorithm2e}

\usepackage{hyperref}
\hypersetup{colorlinks,linkcolor={blue},citecolor={blue},urlcolor={blue}}    
\usepackage[nameinlink,capitalise]{cleveref}
\usepackage[utf8]{inputenc} 
\usepackage[T1]{fontenc}  
\usepackage{url} 
\usepackage{authblk}
\usepackage{natbib}
\usepackage[title]{appendix}
\newtheorem{definition}{Definition}
\newtheorem{theorem}{Theorem}[section]
\newtheorem*{theorem*}{Theorem}

\newtheorem{lemma}[theorem]{Lemma}
\newtheorem{remark}{Remark}

\DontPrintSemicolon

\makeatletter
\renewenvironment{algomathdisplay}
 {\[}
 {\@endalgocfline\vspace{-\baselineskip}\]\;}
\makeatother
\title{Prediction Inference of Time Series with Standard ReLU Deep Neural Networks}
\author[1]{Kejin Wu\thanks{Corresponding author: kwu8@luc.edu}}
\affil[1]{Department of Mathematics and Statistics, Loyola University Chicago}

\begin{document}
    \maketitle

    \begin{abstract}
We propose a methodology based on the standard ReLU Deep Neural Networks (DNN) to make predictions and quantify their uncertainty. Classically, people rely on linear, non-linear, or non-parametric kernel methods to fit and then predict the time series. As the universal approximation ability was revealed for DNN, its application has become more and more popular for prediction tasks in various scientific areas. However, the corresponding uncertainty quantification has not been studied thoroughly. Particularly, the uncertainty in prediction will consist of two parts: (1) the future variability; (2) the estimation variability within training data. To capture both variabilities, we build the so-called pertinent prediction interval (PPI) with the DNN model estimator. We first explore the consistency property of the DNN estimator with $\beta$-mixing dependent data. Subsequently, we show that the implied forward bootstrap series is still $\beta$-mixing and possesses the same stationary distribution as the original time series in probability, which is a key condition to enable the PPI. Lastly, the desired PPI is built after imposing minimal conditions on the limiting distribution of predictive roots. Simulations and real-data analysis are deployed to challenge our approach with standard non-parametric methods.    
    \end{abstract}
{\footnotesize
\noindent \textbf{\textit{Keywords:}} Forecasting, Nonparametric methods, Uncertainty quantification, Deep neural networks
\\
\noindent \textbf{\textit{2020 Mathematics Subject Classification codes:}} 62-08, 62G05, 62G20
}

\section{Introduction}
The study of predicting dependent data, especially (discrete) time series, has a wide range of applications in areas such as finance, economics, healthcare, and engineering. The early recording and modern analysis of time series data could be traced back to the work of \cite{yule1927vii}, who investigated the sunspot numbers. Starting from that, various time series modeling methods have been established. The most popular and standard statistical method to model a (univariate) time series is based on the Autoregressive Moving-average (ARMA) model, which can be identified, fitted, and verified in the Box-Jenkins procedure \citep{box2015time}. However, the ARMA model is far from being satisfactory in practice since it is only designed to handle linear time series. As pointed out by the work of \cite{de1992some} and \cite{tjostheim1994non}, there are various occasions when prior knowledge indicates the data-generating process is in a non-linear form. For example, the ARCH and GARCH models are in a non-linear format and have been revealed to mimic the volatility clustering phenomenon successfully; see \cite{politis2009financial} for more cases of non-linear time series models. 

Although the Non-linear Autoregressive (NLAR) model seems to be a better choice due to its flexibility, the estimation procedure is not trivial. Even for the lag order selection, it is not clear how the participants should make a decision based on the classical Auto-correlation and Partial Auto-correlation functions (ACF and PACF). Let alone the difficulty of fitting the NLAR model. Another challenge hinges on the prediction inference based on the NLAR model, especially for the multi-step ahead forecasting. Unlike the prediction practice with the linear time series model, the simple iterative prediction procedure applied with the non-linear time series model renders a biased optimal $L_2$ multi-step ahead point prediction; see \cite{wu2024bootstrap} for an illustration. To overcome the multi-step prediction difficulty of the NLAR model, there are attempts through various numerical approaches; see the work of \cite{clements1996multi, guo1999multi, zhang1998forecasting, lee2003new} for partial references. However, these methods do not meet the demand to derive an optimal $L_1$ point prediction that plays an important role, especially for areas of econometrics, climate modeling, and water resources management, in which data might not possess a finite second moment, so that the optimization in $L_2$ loss is vacuous. 

Recently, \cite{politis2015model} and \cite{pan2016bootstrap} proposed a forward bootstrap prediction method that could serve to find the one-step ahead optimal $L_2$ or $L_1$ point prediction and prediction interval for NLAR time series models. Upon this, \cite{politis2023multi, wu2024bootstrap} extended such a forward bootstrap prediction method to perform multi-step ahead prediction inference of NLAR models in the parametric or non-parametric approach. Remarkably, this bootstrap approach gives a gateway to build a superior prediction interval, which captures the estimation variability that comes from the model fitting. This so-called pertinent prediction interval (PPI) suffers less from the undercoverage issue in finite sample cases compared to the classical quantile prediction interval or symmetric prediction interval that is centered around some meaningful point; see \cref{Sec:PrePI} for the formal definition of PPI. In addition to the application of PPI in the time series forecasting context, \cite{wang2021model} and \cite{wu2024deep} applied the PPI to quantify the prediction accuracy in the regression context. In short, although the PPI is asymptotically equivalent to classical PIs, the PPI can achieve the same coverage level with shorter data compared to its counterpart. This is also the reason that PPI works well in the finite sample cases. 

Recently, Deep Neural Networks (DNN), one of the Machine Learning methods, has attracted increasing attention due to its strong flexibility in estimation. The initial application of this powerful Neural Networks estimator can be traced to the 1940s. The work of \cite{mcculloch1943logical} showed that even simple types of neural networks could, in principle, compute any arithmetic or logical function. Till now, the universal approximation property of shallow neural networks or DNNs with various classes of activation functions has been shown, i.e., it can be used to approximate functions with the required smoothness condition in any desired degree of accuracy; see the work of \cite{cybenko1989approximation,hornik1989multilayer, yarotsky2018optimal, yarotsky2020phase} for references. In practice, people usually consider the DNN rather than the shallow counterpart due to its better practical performance. Also, to avoid the gradient vanishing problem, people replace the standard sigmoid activation function with the ReLU activation function. Interestingly, it seems that the DNN can conquer the curse of dimensionality as they have great performance for some high-dimensional tasks, e.g., image processing; researchers are actively explaining this phenomenon; see \cite{bauer2019deep,kohler2019estimation} for related work. Moreover, \cite{nakada2020adaptive} argued that the DNN may adapt to the intrinsic low dimensionality of the data for broader cases compared to standard kernel methods. Our simulation studies in \cref{Sec:Simulation} further support the claim that the DNN is more robust against the curse of dimensionality. 

However, the universal approximation property is regarding the optimal estimation ability of neural networks, which is not achievable when we need to train the neural networks based on a specific loss function with data at hand. Recently, the non-asymptotic error bound of the least squares DNN estimator to estimate regression functions has been widely studied under different settings; see the work of \cite{bauer2019deep,schmidt2019deep, nakada2020adaptive, schmidt2020nonparametric, farrell2021deep, jiao2023deep, wu2025scalable} for some references; the complete estimation inference of DNN estimators is still in development. From a predictive perspective, people start to explore the prediction of time series in a machine learning way with DNN; see, for example, \cite{chakraborty1992forecasting} applied a simple shallow neural network to predict flour prices in different cities; \cite{fu2016using} took the Long Short Term Memory to predict traffic flow; \cite{wang2019multiple} applied convolutional neural networks to predict multivariate time series of traffic and solar-energy; see \cite{han2019review} for a comprehensive review. However, these works usually treat the DNN as a ``black box'' and ignore statistical analysis of the prediction results. Also, these works omit the accuracy quantification of the resulting prediction, or they build standard PIs without any efforts to capture the estimation variability of the training DNN.

Motivated by the forward bootstrap prediction procedure, we attempt to make statistical prediction inference for NLAR time series with DNN. More specifically, we show that the multi-step ahead point $L_1$ and $L_2$ predictions are consistent with the oracle conditional mean and median of the future time series. Also, we prove that the naive quantile prediction interval based on DNN is asymptotically valid, with the corresponding definition given in  \cref{Sec:PrePI}. Moreover, with minimal assumptions, we show the possibility of building PPI based on DNN estimators for NLAR models. In other words, the proposed forward bootstrap prediction procedure with DNN kills two birds with one stone: (1) it ensures the consistency of the point predictions with more flexible and powerful DNN estimators; (2) it grants the ability to capture the estimation variability of DNN training. We should mention that we focus on a standard fully connected ReLU DNN in this paper. More precisely, we consider the fully connected ReLU Multi-layer Perceptron. This prediction approach shall be adaptable to other types of DNN under the necessary assumptions. 

The paper is organized as follows. The structure of the standard ReLU DNN and related constraints are illustrated in \cref{Sec:StructureDNN}. The point prediction and corresponding PI according to the forward bootstrap prediction procedure are constructed in \cref{Sec:PrePI}. Simulation and real-world data studies are deployed in \cref{Sec:Simulation,Sec:realdata}. \cref{Sec:conclusion} concludes this paper.

\section{Standard ReLU Deep Neural Networks}\label{Sec:StructureDNN}
To make this paper self-contained and provide convenience to readers who are not familiar with Deep Neural Networks estimators, we give a short introduction to the structure of DNN here. The DNN considered in our method is also called a fully connected Multi-layer Perceptron (MLP). In short, the DNN can be viewed as a parameterized family of functions. Its structure mainly depends on the input dimension $d$, depth $L\in \mathbb{N}$, width $\bm{H} \in \mathbb{N}^{L}$, and the output dimension. The depth $L$ describes how many hidden layers a DNN possesses; the width $\bm{H} = (H_1,\ldots,H_L)$ represents the number of neurons in each hidden layer. The fully connected property indicates that each hidden neuron receives information from all hidden neurons at the previous hidden layer in a functional way. 

Formally, if we let $\bm{u}_l = (u_{l,1},\ldots,u_{l,H_l})^{T}$ to represent all number of neurons at the $l$-th hidden layer for $l = 0, \ldots, L+1$; here, $\bm{u}_0$ represents the input vector $(x_1,\ldots,x_d)^{T}$ and $\bm{u}_{L+1}$ is the output. Therefore, we can pretend that the input layer and the output layer are the $0$-th and $(L+1)$-th hidden layers, respectively. Then, $u_{l,i} = \sigma(\bm{u}_{l-1}^{T}\bm{w}_{l-1,i} + b_{l-1,i})$ for $l = 1,\ldots,L$ and $i = 1,\ldots,H_l$; here $\bm{w}_{l-1,i}\in \mathbb{R}^{H_{l-1}}$ is the weight vector which connects the $(l-1)$-th hidden layer and the neuron $u_{l,i}$; $b_{l-1,i}\in \mathbb{R}$ is the corresponding intercept term; $\sigma(\cdot)$ is the so-called activation function and we take the ReLU function in this paper. To get the output layer, we just take $u_{L+1,i} = \bm{u}_{L}^{T}\bm{w}_{L,i} + b_{L,i}$ for $i = 1,\ldots,H_{L+1}$; here $H_{L+1}$ is equal to the output dimension. To express the functionality of the DNN in a more concise way, we can stack $\{\bm{w}_{l-1,i}^{T}\}_{i=1}^{H_{l}}$ by row to get $\bm{W}_{l-1}\in\mathbb{R}^{H_{l}}\times \mathbb{R}^{H_{l-1}}$ and collect $\{b_{l-1,i}\}_{i=1}^{H_l}$ to be a vector $\bm{b}_{l-1}$ for $l=1,\ldots, L+1$. Subsequently, we can treat the DNN as a function that takes the input $\bm{x}$ and returns the output in the following way:
\begin{equation*}
    f_{\text{DNN}}(\bm{x}) =  \bm{W}_{L}(\sigma(\bm{W}_{L-1}(\cdots \sigma(\bm{W}_{2}\sigma(\bm{W}_{1}\sigma(\bm{W}_0\bm{x} + \bm{b}_{0}) + \bm{b}_1)  + \bm{b}_2)\cdots) +  \bm{b}_{L-1}) + \bm{b}_{L}.
\end{equation*}
We can understand that the function $f_{\text{DNN}}(\bm{x})$ maps $\bm{x}$ to the $1$-st hidden layer and then map the $1$-st hidden to the $2$-nd hidden layer and so on iteratively with weights $\{\bm{W}_l\}_{l=0}^{L}$, $\{\bm{b}_l\}_{l=0}^{L}$ and the activation function $\sigma(\cdot)$. In this work, we solely consider the ReLU activation function, which is praised by users for robustness to gradient vanishing. For reference, a simple fully connected DNN structure is presented in \hyperref[Appendxi:IllustrationStructure]{Appendix: A}. Throughout this paper, we will consider a class of DNN defined as follows:

\begin{definition}\label{Definition:DNNClass}
The class of DNN considered in this paper is denoted by:
\begin{equation}
    \begin{split}
        \mathcal{N}_n &= \mathcal{N}(L_n,H_n,2B_n) \\
        &:= \{f_{\text{DNN}}(\bm{x}):~\text{with depth at most}~L_n~\text{and equal width at most}~H_n; \| f_{\text{DNN}}(\bm{x})\|_{\infty}\leq 2B_n \},
    \end{split}
\end{equation}
where $L_n,H_n$ and $B_n$ can grow up with the sample size $n$ converges to $\infty$. The rate of the growth will be discussed later.
\end{definition}

\begin{remark}[The relation between fully connected DNN and general DNN]
    Since a fully connected DNN is less restricted than a sparse general DNN — any general DNN can be embedded into a fully connected MLP by zeroing the missing connections — the fully connected DNN's function class is larger, and its estimation-error bound is also larger. Consequently, an estimation-error bound established for the fully connected DNN also controls the error of any general DNN embedded within it, so the asymptotic results of the fully connected DNN carry over to a general DNN estimator.
\end{remark}

\begin{remark}[The number of parameters of a DNN and the sample size]
    We can compute the total number of parameters in a DNN by the formula $W =  \sum_{i=0}^{L}(H_{i}\cdot H_{i+1} + H_{i+1})$. In the big-data era, machine learning researchers tend to apply a DNN model with $W$ being larger than the sample size to fit data. This action violates the classical statistical rule that the number of parameters should not exceed the sample size; otherwise, overfitting will spoil the generation performance of the model. However, \cite{belkin2019reconciling,belkin2020two} discovered the interesting double-descent phenomenon appearing in the large model estimation. More specifically, they found that the mean squared error of the model estimator increases and then decreases when the number of parameters keeps increasing in the over-parametrized region. However, as far as we know, the formal analysis for the estimation inference of DNN-type estimators usually requires a large enough sample size, e.g., $n>Pdim(\mathcal{N}_n)$, where $Pdim(\mathcal{N}_n)$ is the pseudo-dimension of the DNN function class $\mathcal{N}_n$ defined by \cite{bartlett2019nearly}, which is at least of the same order of $W$. In this paper, we focus on the under-parametrized situations. In addition, we found that the DNN estimator is more robust to the curse of dimensionality than the standard kernel estimators. 
\end{remark}

\section{Point Predictions and Prediction Intervals with DNN}\label{Sec:PrePI}

In this section, the detailed algorithm to construct the point predictions and prediction intervals will be given. Also, the asymptotic validity of the naive quantile prediction interval (QPI) will be established. Subsequently, the pertinent prediction interval (PPI) will be verified with minimal assumptions. Throughout this paper, we consider the following data-generating process
\begin{equation}\label{Eq:DGP}
    Y_{t} = f_0(\bm{Y}_{t-1}) + \epsilon_t~\text{for}~t = 1,\ldots,n.
\end{equation}
\cref{Eq:DGP} is a general autoregressive model; where $\epsilon_t$ are the $i.i.d.$ innovation terms for $t=1,\ldots,n$ with $\mathbb{E}\left|\epsilon_1\right|<\infty$; $\bm{Y}_{t-1}$ stands for $\{Y_{t-1}, \ldots, Y_{t-d}\}$; $d$ is the order; $f_0$ is the mean function of this general autoregressive model. A $j$-step $(1-\alpha)\times100\%$ $\text{QPI}_{n+j}$ is called asymptotically (conditional) valid if $$\mathbb{P}(Y_{n+j}\in \text{QPI}_{n+j}|\bar{Y}_{n}) \to 1-\alpha~\text{as}~ n\to \infty,$$
where $\bar{Y}_{n}$ represents all observations $(Y_1,\ldots, Y_n)$ throughout this paper. In addition, we make three basic assumptions:
\begin{itemize}
    \item A1. The target mean function $f_0$ lies in the Hölder space $C^{k, \alpha}\left([-B_n,B_n]^d\right)$, which is the space of $k$ times continuously differentiable functions on $[-B_n,B_n]^d$ having a finite norm defined by

    \begin{equation*}
        \begin{split}
           &\|f\|_{C^{k, \alpha}\left([-B_n,B_n]^d\right)}\\
           &=\max \left\{\max _{\bm{k}:|\bm{k}| \leq k} \max _{\bm{x} \in[-B_n,B_n]^d}\left|D^{\bm{k}} f(\bm{x})\right|, \max _{\bm{k}:|\bm{k}|=k} \sup _{\substack{\bm{x}, \bm{y} \in[-B_n,B_n]^d \\ \bm{x} \neq \bm{y}}} \frac{\left|D^{\bm{k}} f(\bm{x})-D^{\bm{k}} f(\bm{y})\right|}{\|\bm{x}-\bm{y}\|^\alpha}\right\},
        \end{split}
    \end{equation*}
where the smoothness order is $r=k+\alpha$ with an integer $k \geq 0$ and $0<\alpha \leq 1$. In particular, we define the function space $\mathcal{F}_{r,[-B_n,B_n]^d} := \{f:  \|f\|_{C^{k, \alpha}\left([-B_n,B_n]^d\right)}< \infty~\text{and}~\|f\|_\infty\leq B_n\}$ and $f_0\in\mathcal{F}_{r,[-B_n,B_n]^d}$.


\item A2. The whole series $\{Y_t\}$ is strictly stationary and geometrically $\beta$-mixing, i.e., $\beta(j) \leq C_\beta^{'} e^{-C_\beta j}$ for some $C_\beta, C_\beta^{\prime}>0$. In addition, we denote the marginal stationary distribution of $\{Y_t\}$ by $\mathbb{P}_{\pi}$ and we assume it is absolutely continuous. 

\item A3. $\lim _{n \rightarrow \infty}\log(n)\cdot \mathbb{P}_{\pi}\left(\max _{t \in\{1, \ldots, n\}}\left|Y_t\right| \geq B_n\right)=0$.
\end{itemize}
The assumption A1 describes the smoothness of the target function $f_0\in \mathcal{F}_{r,[-B_n,B_n]^d}$. We assume the domain and the $\infty$-norm of the function value can increase with the sample size $n$ in an appropriate sequence $B_n$. A similar assumption is made in \cite{yarotsky2020phase} with a constrained domain $[0,1]^d$ to show the oracle approximation ability of DNN. We make a more general assumption. Compared to assumption 1 of \cite{brown2024statistical}, where $\|D^{\bm{k}}f\|$ is required to be less than 1, our assumption A1 does not impose such a constraint, and we just need it to be finite. Moreover, the proof of \cite{brown2024statistical} is based on the optimal estimation results of DNN from \cite{yarotsky2017error}. With the advanced results from \cite{yarotsky2020phase}, we present \cref{Lemma:theEstimationability} to illustrate the estimation ability of DNN estimation, which is a key result that stands on its own interest and is fundamental for the left proof.

Assumption A2 is a common condition in time series analysis. One sufficient condition to guarantee A2 is given in \cref{Subsec:PPI}. For more related discussion, we refer readers to the work of \cite{mokkadem1988mixing, chen2000geometric,carrasco2002mixing} for the requirement on the $f_0$ and $\epsilon$ term to obtain strictly stationary and geometrically $\beta$-mixing series. 

The assumption A3 serves as a kind of moment condition for dependent data. Similarly, \cite{farrell2021deep} required their assumption 1, where the response variable in the regression problem is assumed to be in a compact set; \cite{chen2015optimal} required their assumption 2, where the conditional expectation of the squared error term is assumed to be uniformly bounded. We take assumption A3, which is also assumed in \cite{brown2024statistical}. One of the sufficient conditions to obtain assumption A3 is assuming A2 and $\lim _{n \rightarrow \infty} n \mathbb{P}_{\pi}\left(\left|Y_1\right|>B_n\right)=0$.



\begin{lemma}\label{Lemma:theEstimationability}
Under assumption A1, we consider the DNN estimator that comes from the class $\mathcal{N}_n(L_n,H_n,$ $(2B_n)^{r+1})$ defined in (\ref{Definition:DNNClass}). When  $n$ is sufficiently large,
we have 
$$\forall f_0 \in \mathcal{F}_{r,[-B_n,B_n]^d},~\exists f^*_{DNN} \in \mathcal{N}(L_n,H_n,(2B_n)^{r+1}),~s.t.,~\|f^*_{DNN} - f_0 \|_{\infty} \leq (2B_n)^rn^{-\eta\frac{r}{d}};$$
where $L_n = O(\log n); H_n = O(n^{\eta}\log n)$ and $\eta \in (0,1)$ is an appropriate constant. $f^*_{DNN}$ and $ f_0$ have the same domain. The rate of $\infty$-norm error converges to 0, depending on the behavior of the sequence $(2B_n)^rn^{-\eta\frac{r}{d}}$ as $n \to \infty$.
\end{lemma}

\begin{remark}
    \cref{Lemma:theEstimationability} reveals the optimal approximation ability of the standard fully connected feedforward DNN. The optimal $\infty$-norm error depends on how many samples we have and how smooth the underlying true function is. We should mention that the optimal approximation ability is usually not accessible in practice since we need to train the empirical optimal DNN estimator based on some specific loss function with samples. Also, we can see $B_n$ should increase slowly as $n$ increases; otherwise the $\infty$-norm error will not converge to 0.
\end{remark}


In practice, solely \cref{Lemma:theEstimationability} can not guide the empirical estimation performance. Instead of the optimal error bound, we need an error bound to measure the DNN estimator trained with some specific loss function. Furthermore, we can think of the DNN estimator as a kind of sieve estimator in which the approximating parameter space gets larger and larger. More specifically, according to the notation in \cite{brown2024statistical}, we call $\hat{f}_n$ as the approximate sieve estimator if 
\begin{equation}\label{Eq:approximate_sieve}
Q_n\left(\hat{f}_n\right) \leq \inf _{f \in \mathcal{F}_n} Q_n(f)+\theta_n;
\end{equation}
where $Q_n(\cdot)$ is some appropriate empirical loss function, e.g., 
\begin{equation}\label{Eq:approximate_sieveall}
    Q_n(f) = \frac{1}{n}\sum_{t=1}^{n} q(\bm{Z_t},f) =\frac{1}{n} \sum_{t=1}^{n}(Y_{t} - f(\bm{Y}_{t-1}))^2
\end{equation}
for $t = 1, \ldots, n$ taken in our paper, where $\bm{Z_t} := (Y_{t},\bm{Y}^{'}_{t-1})^{'}$; $\left\{\mathcal{F}_n\right\}_{n \in \mathbb{N}}$ is a sequence of sieve spaces such that $\mathcal{F}_n \subseteq \mathcal{F}_{n+1}$, and $\sup _{f \in \mathcal{F}_n}\|f\|_{\infty}<\infty$; $\theta_n: \Omega \rightarrow[0, \infty)$ is a random variable, such that $\theta_n=\mathrm{o}_P(1)$. In particular, when $\theta_n\equiv 0$, $\hat{f}_n$ is the exact sieve estimator. Later, we will show the sieve spaces can be taken as the DNN class $\mathcal{N}_n$.

To incorporate the expanding range of the target function defined in A1, we give a variant of Theorem 2 from \cite{brown2024statistical} in \cref{Lemma:variantofTheorem2} with the following additional conditions.

\begin{itemize}
    \item B1. $\forall~n \in \mathbb{N}$, there exists a non-stochastic $\tilde{f}_n \in \mathcal{F}_n$ where $\tilde{\epsilon}_n:=\left\|\tilde{f}_n-f_0\right\|_{\infty}$ is such that $\lim _{n \rightarrow \infty} \tilde{\epsilon}_n(\log n)(\log \log n)=0$.
    \item B2. $\exists~ C_3, C_4>0$ s.t. $C_3 \leq 1$, and $\forall~n, f \in \mathcal{F}_n$,
$$
C_3\left\|f-f_0\right\|_{\mathcal{L}^2}^2 \leq \mathbb{E}\left[Q_n(f)\right]-\mathbb{E}\left[Q_n\left(f_0\right)\right] \leq C_4\left\|f-f_0\right\|_{\mathcal{L}^2}^2.
$$
\item B3. For $\mathcal{F}_n$, $\operatorname{Pdim}\left(\mathcal{F}_n\right) \geq 1$ for all $n$, and for some non-decreasing sequence $\left\{D_n\right\}_{n \in \mathbb{N}}$ such that $D_1 \geq 2, \sup _{f \in \mathcal{F}_n}\|f\|_{\infty} \leq D_n<\infty$ for each $n$, and

$$
\lim _{n \rightarrow \infty} \frac{D_n}{\sqrt{n}}\left[\sqrt{Pdim\left(\mathcal{F}_n\right) \log (n)}+\sqrt{\log \log (n)}\right]=0.
$$

\item B4. $\exists~m_n(\cdot)$ s.t., $\forall~f, f^{\prime} \in \mathcal{F}_n,$
$\left|q(\boldsymbol{z}, f(\boldsymbol{z}))-q\left(\boldsymbol{z}, f^{\prime}(\boldsymbol{z})\right)\right| \leq m_n(\boldsymbol{z})\left|f(\boldsymbol{z})-f^{\prime}(\boldsymbol{z})\right|
$, and $\exists~C_5 \geq 1$ and a strictly positive sequence $\left\{\mu_n\right\}_{n \in \mathbb{N}}$ s.t.,
$$ \sup _{f \in\left\{\mathcal{F}_n \cup\left\{f_0\right\}\right\}} \mathbb{E}\left[\left|q\left(\boldsymbol{Z}_t, f\left(\boldsymbol{Z}_t\right)\right)\right| \mathbbm{1}_{\left\{m_n\left(\boldsymbol{Z}_t\right) \geq C_5 D_n\right\}}\right]\leq \mu_n,$$
and $\lim _{n \rightarrow \infty} \mu_n=0$, where $q(\cdot, \cdot)$ and $\bm{Z}_t$ are defined in \cref{Eq:approximate_sieveall}.
\end{itemize}

For the additional conditions, B1 stands for the representation ability of the class of the estimator, i.e., the best estimation performance of an estimator from $\mathcal{F}_n$; B2 is used to make the main decomposition for the proof. It holds for $Q_n$ taken as \cref{Eq:approximate_sieveall}; B3 requires the pseudo dimension to be in small order w.r.t. $n$; B4 adds constraints on the tail property of the series. 

\begin{lemma}[Variant of Theorem 2 in \cite{brown2024statistical}]\label{Lemma:variantofTheorem2} Under assumptions A2, B1 - B4, if $\hat{f}$ is measurable and satisfying \cref{Eq:approximate_sieve}, $n \geq \max \left\{5,2 Pdim\left(\mathcal{F}_n\right), 16 D_n^2 / \log (n)\right\}$, $\max \{2, $$\frac{\tilde{\epsilon}_n(\log n)(\log \log n)}{D_n}\}$ $<a \leq n / 2,$ and $\quad \sqrt{\delta} \geq \frac{\tilde{\epsilon}_n \sqrt{n}}{D_n a-\tilde{\epsilon}_n(\log n)(\log \log n)}$ where $a \in \mathbb{N}$ and $\delta >0$, then $\exists$ a constant $C>0$ which only depends on $C_3$, $C_4$, $C_5$, $C_\beta$ and $C_\beta^{'}$, 
$$P\left(\left\|\hat{f}_n-f_0\right\|_{\mathcal{L}^2} \leq C \epsilon_n(\delta, a)\right) \geq 1-e^{-\delta}-2 \log (n)\left[\frac{n \beta(a)}{a}+2 P\left(\max _{t \in\{1, \ldots, n\}} m_n\left(\boldsymbol{Z}_t\right) \geq C_5 D_n\right)\right];$$
where $\epsilon_n(\delta, a):=D_n \sqrt{\frac{a}{n}}\left[\sqrt{Pdim\left(\mathcal{F}_n\right) \log (n)}+\sqrt{\log \log (n)+\delta}\right]+\sqrt{\tilde{\epsilon}_n^2+\mu_n+\theta_n}$ and $\|f_0\|_{\infty}\leq D_n$.
\end{lemma}

\begin{remark}[Compared to Theorem 2 in \cite{brown2024statistical}]
In our lemma, the range of the target function $f_0$ can be grown slowly with an upper bound $B_n$. Subsequently, we will show that the function class defined in assumption B1 can be replaced with the DNN class $\mathcal{N}_n$. Also, the sequence $D_n$ in assumption B3 can be taken as an appropriate sequence $B_n$. Then, we can get the error bound for the DNN estimator. Integrating \cref{Lemma:theEstimationability}, our DNN estimator possesses an expanding domain, i.e., $[-B_n, B_n]$, which is more general than the estimator defined in \cite{farrell2021deep,brown2024statistical}.
\end{remark}

With \cref{Lemma:variantofTheorem2}, we can show $L_2$ estimation error bound of one DNN estimator that comes from the class $\mathcal{N}_n$. The results are presented in \cref{Theorem:esterrorDNN}.

\begin{theorem}\label{Theorem:esterrorDNN}
Under assumptions A1 to A3, taking $B_n \asymp n^{K_B}$ for $K_B\in[0,\frac{1}{2(r+d)+2})$, and setting $L_n \asymp \log(n)$ and $H_n \asymp n^{\left(\frac{d}{r+d}\right)\left(1 / 2-K_B\right)} \log (n)$\footnote{We denote $x_t \asymp y_t$ if $x_t \leq c_1 y_t$ and $y_t \leq c_2 x_t$ for two constants $c_1$ and $c_2$ when $t$ is sufficiently large.}. Then, if the DNN estimator satisfies \cref{Eq:approximate_sieve} and $Q_n(\cdot)$  is taken as \cref{Eq:approximate_sieveall}, and

$$
\epsilon_n=C\left(n^{rK_B - \frac{r}{r+d}(1/2 - K_B)}\log^4(n) + \sqrt{\mu_n+\theta_n} \right),
$$
where $\mu_n := max \left\{6 \mathbb{E}\left[Y_t^2 \mathbbm{1}_{\left\{\left|Y_t\right| \geq B_n\right\}}\right], n^{-1}\right\}$ and $C>0$ is a constant independent of $n$, then for all $n$ sufficiently large s.t., $n \geq \max \{5,2 Pdim\left(\mathcal{F}_n\right), $ $16 B_n^2 / \log (n)\}$ and $$\max \left\{2, \frac{\tilde{\epsilon}_n(\log n)(\log \log n)}{B_n}\right\}<a:= \left\lceil \log ^2(n)\right\rceil \leq n / 2,$$
where $\tilde{\epsilon}_n = O(\left(2 B_n\right)^r n^{-\left(\frac{r}{r+d}\right)\left(1 / 2-K_B\right)})$ shown in the proof. 

Then, we have 
$
\left\|\hat{f}_n-f_0\right\|_{\mathcal{L}^2} \leq  \epsilon_n,
$
with probability at least 
$$1 - e^{-C_\delta n^{2rK_B+2\left(\frac{d}{r+d}\right)\left(1 / 2-K_B\right)}\log^6(n)} - 2 \log (n)\left[\frac{C_\beta^{'} n^{1-C_\beta\log n}}{\log^2(n)}+2 P\left(\max _{t \in\{1, \ldots, n\}} Y_t\geq 2B_n\right)\right];$$ $C_\delta$ is another constant independent with $n$.
\end{theorem}
In short, the proof of \cref{Theorem:esterrorDNN} relies on \cref{Lemma:variantofTheorem2}, where conditions B1 to B4 are checked one by one with only assumptions A1 to A3. Here, the target function $f_0$ comes from $\mathcal{F}_{r,[-B_n,B_n]^d}$.
\begin{remark}
    We should mention that $K_B$ is small when $r + d$ is large, which means the range $[-B_n, B_n]$ will expand slowly. On the other hand, if $K_B = 0$, we achieve the fastest convergence rate shown in $\epsilon_n$, i.e., $n^{-\frac{r}{2(r+d)}}$, which is also the rate that can be achieved in a common setting by \cite{farrell2021deep,brown2024statistical}. In other words, when $K_B$ is large, the convergence rate is low, and it is hard to estimate $f_0$. We conjecture this phenomenon is reasonable due to the difficulty of estimating a function whose domain and range are expanding fast. 
    \end{remark}

With the $L_2$ error bound of $\hat{f}_n$ on estimating $f_0$ illustrated in \cref{Theorem:esterrorDNN}, we attempt to build the pointwise convergence between $f_0$ and $\hat{f}_n$, which requires additional assumptions to control the continuity of $\hat{f}_n$. Thus, we assume that:
\begin{itemize}
    \item C1. There exist $\mathcal{L}_n<\infty$ and events $\mathcal{E}_n$ with $P\left(\mathcal{E}_n\right) \rightarrow 1$ such that on $\mathcal{E}_n$,

$$
\left|\hat{f}_n(\bm{y})-\hat{f}_n\left(\bm{y}^{\prime}\right)\right| \leq \mathcal{L}_n\left\|\bm{y}-\bm{y}^{\prime}\right\| \quad \forall \bm{y}, \bm{y}^{\prime} \in [-B_n, B_n]^d .
$$
    \item C2. The joint density of $\bm{Y}_{t-1}$, namely $\pi(\bm{y}_{t-1})$, has a positive lower bound, i.e., $\pi(\bm{y}_{t-1})\geq C_{\pi}>0$ for all $\bm{y}_{t-1}$ in its domain, where $C_{\pi}$ is a constant. 
    \item C3. $\epsilon_n$ defined in \cref{Theorem:esterrorDNN} and the Lipschitz constant of $g:= \hat{f} - f_0$, namely $\tilde{\mathcal{L}}_n$ satisfy $\frac{\tilde{\mathcal{L}}_n^d  \epsilon_n^2}{C_{\pi}} \longrightarrow 0$ as $r \geq 1$ and $n\to\infty$.
\end{itemize}
\begin{remark}
    For assumption C1, it holds with $\mathcal{L}_n=\prod_{l=0}^L\left\|\bm{W}_{l}\right\|_{\text{op}}$, where $\|\cdot\|_{\text{op}}$ is the operator norm for a matrix. Thus, to get a finite $L_n$, we just need the operator norm of each $\bm{W}_{l}$ to be finite. This can be achieved by applying the parameter truncation trick in the training procedure. For assumption C3, the condition $r\geq 1$ is used to guarantee that the true function $f_0$ is Lipschitz continuous. Also, it is possible that the Lipschitz constant $\tilde{L}_n$ increases as $n$ increases as long as $\frac{\tilde{\mathcal{L}}_n^d \epsilon_n^2}{C_{\pi}}$ converges to 0.  
\end{remark}

Under these additional assumptions, it is easy to find that the DNN estimator $\hat{f}_n$ is consistent with the true function $f_0$ for $\forall ~ \bm{y}_{t-1}\in[-B_n, B_n]^d$. We summarize this result in \cref{Prop:consistency} as follows. 
\begin{lemma}\label{Prop:consistency}
With all conditions assumed in \cref{Theorem:esterrorDNN} and C1 to C3, we have 
$$
\|\hat{f}_n(\bm{y}_{t-1}) - f_0(\bm{y}_{t-1})\|_{\infty}
 \overset{p}{\to} 0,~\text{as}~n\to\infty
$$
for $\bm{y}_{t-1}\in[-B_n, B_n]^d$; where $\hat{f}_n$ is the DNN estimator defined in \cref{Theorem:esterrorDNN} and $f_0$ is the true function defined by A1. 
\end{lemma}
\begin{remark}[Local version of \cref{Prop:consistency}]
In \cref{Prop:consistency}, we have the uniformly pointwise convergence of $\hat f_n$ to $f_0$ in probability for all $\bm{y}_{t-1}$ in its domain. This result requires the global Lipschitz control as assumptions C1 and C3. We can weaken this global condition and derive the local version of \cref{Prop:consistency}, i.e., the pointwise convergence holds for some point $\bm{y}^*_{t-1}$ s.t. the Lipschitz continuous condition is satisfied around this specific point $\bm{y}^*_{t-1}$. In particular, we can take the local version of assumptions C1 and C2 on a ball around $\bm{y}^*_{t-1}$ and require a similar assumption C3.
    
\end{remark}

\subsection{Point predictions and naive prediction interval}
Equipped with the estimation inference of DNN estimators developed in \cref{Lemma:variantofTheorem2}, \cref{Theorem:esterrorDNN} and \cref{Prop:consistency}, we are ready to propose the algorithm to build a naive Prediction Interval (QPI) and Pertinent Prediction Interval (PPI) with DNN estimators. First, we present the procedure to obtain optimal $L_2$ point prediction and the QPI in \cref{alg:QPI}, where we assume we have observations $(Y_{-d+1},\ldots,Y_0,Y_1, \ldots, Y_n)$ to get $n$ residuals.

\begin{algorithm}[htbp]
        \caption{Multi-step ahead point predictions and the corresponding QPI.}\label{alg:QPI}
        \KwData{Observations $(Y_{-d+1},\ldots,Y_0,Y_1, \ldots, Y_n)$.}
        \KwResult{Optimal $L_2$ point predictions and $100(1-\alpha)\%$ $\text{QPI}_{n+j}(\alpha)$ for $Y_{n+j}$; $j = 1,\ldots, J$.}
        Estimate $f_0$ by $\hat{f}_n$ with DNN or Kernel estimators.\;
        \For{$t$ in $1, \ldots, n$}{
            Compute fitted value $\hat{Y}_t \leftarrow \hat{f}_n(\bm{Y}_{t-1})$ and fitted residual $\hat{\epsilon}_t \leftarrow Y_t - \hat{Y}_t$. \; 
        }
        Compute mean of fitted residuals $\overline{\hat{\epsilon}} \leftarrow \frac{1}{n}\sum_{t=1}^{n} \hat{\epsilon}_t$. \;
        Center the residuals $r_t \leftarrow \hat{\epsilon}_t - \overline{\hat{\epsilon}}$ and let $\widehat{F}_{cr}$ denote the empirical distribution of $r_t$ for $t = 1, \ldots, n$. \;
        \For{$s$ in $1, \ldots, S$}{
            Draw $\epsilon_{n+1}^*,\ldots,\epsilon_{n+J}^* \overset{\text{i.i.d.}}{\sim} \hat{p}_{\epsilon}$. \;
            Generate $Y_{n+1}^*,\ldots,Y_{n+J}^*$ via 
                \begin{algomathdisplay}
                    \begin{aligned}
                        Y_t^* \leftarrow \hat{f}_n(\bm{Y}_{t-1}^*) + \epsilon_t^*;~\textbf{for~}t = n+1, \ldots, n+J.
                    \end{aligned} 
                \end{algomathdisplay}
            Obtain the $s$-th bootstrap series $\{Y^{(s),*}_{n+1}, \ldots, Y^{(s),*}_{n+J}\}$. \;
        }
        Set the $h$-th step ahead optimal $L_2$ point prediction as
        \begin{algomathdisplay}
            \begin{aligned}
                \hat{Y}_{n+j} = \frac{1}{S}\sum_{s=1}^S Y^{(s),*}_{n+j};~\text{for}~j = 1, \ldots, J.
            \end{aligned}
        \end{algomathdisplay}
        
        Set the $100(1-\alpha)\%$ QPI for $Y_{n+j}$ as
            \begin{equation*}
               \text{QPI}_{n+j}(\alpha) := [\hat{L}_{n+j}, \hat{U}_{n+j}]; 
            \end{equation*}
            where $\hat{L}_{n+j}$ \text{~and~} $\hat{U}_{n+j}$ \text{~are~} $\alpha/2$ \text{~and~} $1-\alpha/2$ \text{~quantile of~} $\{Y^{(s),*}_{n+j}\}_{s=1}^S$~\text{respectively}; $j = 1,\ldots, J$.
    \end{algorithm}

It is easy to derive the multi-step ahead optimal $L_1$ point predictor by taking $\hat{Y}_{n+j} = \text{Median}\{Y^{(s),*}_{n+j}\}_{s=1}^S;~\text{for}~j = 1, \ldots, J$, which is defined in \cref{alg:QPI}. The statistical validity of the $L_2$ point prediction and QPI returned by \cref{alg:QPI} is analyzed in \cref{Theorem:validaityofQPI}. First, we present a useful \cref{Lemma:uniformconsistencyofCDF} that also stands on its own interest. To simplify the proof, we make one assumption on the density function of $\epsilon_t$ as follows:
\begin{itemize}
    \item A4. The probability density of innovation $\epsilon_t$, namely $p_\epsilon(x)$, is continuously differentiable and satisfies $\sup _x p_\epsilon(x)<\infty$, and $|p_\epsilon(a)-p_\epsilon(b)| \leq L|a-b| $ for a finite constant $L$ and any $a$ and $b$ belong to the domain.
\end{itemize}
We should remark that assumption A4 is similar to assumption 11 made by \cite{wu2024bootstrap}. In addition, assumption 4 of \cite{franke2002properties} and assumption B4 of \cite{franke2004bootstrapping} are also in a similar format to our A4. 

\begin{lemma}\label{Lemma:uniformconsistencyofCDF}
With all conditions assumed in \cref{Theorem:esterrorDNN}, C1 to C3 and A4, the CDF of innovation $F_\epsilon(x)$ can be approximated by the empirical CDF of residuals $\widehat{F}_\epsilon(x)$ in a way:

$$
\sup _x\left|\widehat{F}_\epsilon(x)-F_\epsilon(x)\right| \xrightarrow{p} 0,
$$
where $\widehat{F}_\epsilon(x):=\frac{1}{n} \sum_{i=1}^n \mathbbm{1}_{\hat{\epsilon}_i \leq x} ; \mathbbm{1}(\cdot)$ is the indicator function, and we define the residual $\hat{\epsilon}_i=Y_{t} - \hat{f}_n(\bm{Y}_{t-1})$ for $t=1, \ldots, n$, where $\hat{f}_n$ is the DNN estimator. 
\end{lemma}

\cref{Lemma:uniformconsistencyofCDF} describes the uniform consistency of the empirical distribution $\widehat{F}_\epsilon(x)$ on estimating $F_\epsilon(x)$. With this result, the following \cref{Theorem:validaityofQPI} guarantees the consistency of either $L_2$ or $L_1$ point predictions returned by \cref{alg:QPI} and the asymptotic validity of QPI.

\begin{theorem}\label{Theorem:validaityofQPI}
With all conditions assumed in \cref{Theorem:esterrorDNN}, C1 to C3 and A4, for $\forall~h \geq 1$, we have:
\begin{equation}\label{Eq:validaityofQPI}
    \sup _{|y| \leq B_n}|F_{Y^*_{n+j}|\bar{Y}_n}(y)- F_{Y_{n+j}|\bar{Y}_n}(y)| \xrightarrow{p} 0,
\end{equation}
for each $j$, where $Y_{n+j}^*:=\mathscr{G}(\bar{Y}_n; \hat{\epsilon}_{n+1}^*, \ldots, \hat{\epsilon}_{n+j}^* )$, i.e., a function of $\bar{Y}_n$, $\hat{\epsilon}_{n+1}^*,\ldots,\hat{\epsilon}_{n+j}^*$; this is computed by $Y_{n+k}^*= \hat{f}(\bm Y_{n+k-1}^*) + \hat{\epsilon}^*_{k}$ iteratively for $k=1, \ldots, j; \left\{\hat{\epsilon}_i^*\right\}_{i=n+1}^{n+h}$ are i.i.d. $\sim \widehat{F}_\epsilon$; $F_{Y^*_{n+j}|\bar{Y}_n}(y)$ is the conditional distribution of $j$-step ahead future value in the bootstrap world. In practice, we can further estimate $F_{Y^*_{n+j}|\bar{Y}_n}(y)$ by the empirical version of $\widehat{F}_{Y^*_{n+j}|\bar{Y}_n}(y)$ based on $\{Y^{(s),*}_{n+j}\}_{s=1}^S$ which is defined in \cref{alg:QPI}. Subsequently, we have
\begin{equation}\label{Eq:validaityofQPI2}
    \sup _{|y| \leq B_n}|\widehat{F}_{Y^*_{n+j}|\bar{Y}_n}(y)- F_{Y_{n+j}|\bar{Y}_n}(y)| \xrightarrow{p} 0.
\end{equation}
\end{theorem}

In particular, \cref{Eq:validaityofQPI2} shows the uniform consistency of the conditional CDF estimator. To get the consistency of the point predictions and the asymptotic validity of the QPI, one more step is needed with two more mild conditions
\begin{itemize}
    \item [A5.] $F_{Y_{n+j}|\bar{Y}_n}(y)$ is continuous and strictly increasing w.r.t. any $y$.
    \item[A6.] $B_n$ increases in an appropriate rate s.t., A3 is satisfied and $B_n \sup _{|y| \leq B_n}|\widehat{F}_{Y^*_{n+j}|\bar{Y}_n}(y)- F_{Y_{n+j}|\bar{Y}_n}(y)| = op(1)$.
 \end{itemize}
Then, the property of the point predictions and QPI is summarized as follows in \cref{Lemma:consistenctres}.

\begin{lemma}\label{Lemma:consistenctres}
    With all conditions assumed in \cref{Theorem:validaityofQPI}, A5 and A6, for the domain of $\widehat{F}_{Y^*_{n+j}|\bar{Y}_n}(y)$ is restricted to $|y| \leq B_n$, $L_2$ and any quantile point predictions are consistent. Meanwhile, the QPI is asymptotically valid.
\end{lemma}

\subsection{Pertinent prediction interval}\label{Subsec:PPI}
Although the QPI is asymptotically valid, it suffers from the undercoverage issue when the sample size is small. Therefore, we attempt to build a so-called Pertinent prediction interval (PPI) that can capture the estimation variability so that the coverage rate improves for finite sample cases. We take Definition 2.4 from the work of \cite{pan2016bootstrap} to define the PPI in our context. 

We first consider the predictive roots in the real world and the bootstrap world as follows
\begin{equation*}
    R_{n+j} := Y_{n+j}-\hat{Y}_{n+j}=\epsilon_{n+j}+A_{n+j}~;~R^*_{n+j} := Y_{n+1}^*-\hat{Y}_{n+j}^*=\epsilon_{n+j}^*+A_{n+j}^*,
\end{equation*}
where $\epsilon_{n+j}^*$ is generated from the mean-zero adjusted empirical distribution of the residuals; $A_{n+j}= f_0(\bm{Y}_{n+j-1}) - \hat{f}_n(\bm{Y}_{n+j-1})$ represents the estimation error for any values of $\bm{Y}_{n+j-1}$. Similarly, $A_{n+j}^*=\hat{f}_n(\bm{Y}_{n+j-1}) - \hat{f}_n^*(\bm{Y}_{n+j-1})$; where $\hat{f}^*$ is the DNN estimator with bootstrap time series, see \cref{alg:PPI} for details. Ideally, the distribution of $R^*_{n+j} $ should approximate the distribution of $R_{n+j}$. Then, we can determine the $h$-th step ahead PPI as
\begin{equation}\label{Eq:PPItheoretical}
    [\hat{Y}_{n+j}+q^*(\alpha/2),\hat{Y}_{n+j}+q^*(1-\alpha/2)  ];
\end{equation} 
where $q^*(\alpha/2)$ and $q^*(1-\alpha/2)$ are $\alpha/2$ and $1-\alpha/2$ quantile of $R^*_{n+j}$. In \cref{alg:PPI}, $q^*(\alpha/2)$ and $q^*(1-\alpha/2)$ will be further estimated by the forward bootstrap. 

By construction, the PPI, in principle, can capture both the pure prediction error, i.e., the distribution of $\epsilon_{n+j}$, as well as the estimation error asymptotically. Moreover, due to this specific design, it suffers less from the undercoverage issue for finite sample cases. In other words, the PPI can achieve the same coverage as the QPI but with a smaller sample size. \cite{wang2021model} explained this phenomenon by observing the fact that the predictive root $R_{n+j}$ is a convolution between the distribution of $\epsilon_{n+j}$ and the distribution of estimation error $A_{n+j}$. For a large class of conditional distributions of $Y_{n+j}$, the error of $R^*_{n+j}$ on estimating $R_{n+j}$ is smaller than error of $\widehat{F}^*_{Y^*_{n+j}|\bar{Y}_n}$ on estimating $F_{Y_{n+j}|\bar{Y}_n}$. Thus, PPI is more accurate since it is built based on the predictive root. The formal definition of PPI is given:
\begin{definition}\label{Definition:PPI}
    We say the PI defined in (\ref{Eq:PPItheoretical}) is the $j$-step ahead PPI if the following three conditions are satisfied as $n\to\infty$:
    \begin{itemize}
        \item [D1.] $\sup_x\left|F_\epsilon(x) - \widehat{F}_{\epsilon}(x)\right| \overset{p}{\to} 0$, with $F_\epsilon(x)$ being continuous.
        \item [D2.] $\left|P\left(a_n A_{n+j} \leq a\right)-P^*\left(a_n A_{n+j}^* \leq a\right)\right| \xrightarrow{P} 0$ for some sequence $a_n \rightarrow \infty$, and for all points $a$ where the assumed nontrivial limit of $P\left(a_n A_{n+j} \leq a\right)$ is continuous.
        \item [D3.] $\epsilon_{n+j}^*$ and $A_{n+j}^*$ are independent in the bootstrap world, as their analogs are in the real world.
    \end{itemize}
    where $P^*$ is the conditional probability given $\bar{Y}_n$.
\end{definition}
For condition D1, it is exactly the uniform consistency of the CDF estimation. We have shown that this condition is satisfied in \cref{Lemma:uniformconsistencyofCDF} with the DNN estimator. Condition D3 is met by applying the forward bootstrap algorithm. The crucial issue comes from condition D2, which implies that both $a_n A_{n+j}$ and $a_n A_{n+j}^*$ have the same non-trivial limiting distributions asymptotically. 

We should mention that there are some sparse DNN estimation inference with dependent data \citep{kurisu2025adaptive,kengne2025sparse}; sparse DNN estimation inference with i.i.d. data \citep{suzuki2018adaptivity,tsuji2021estimation,ohn2022nonconvex}; sieve estimator inference for shallow neural networks \citep{shen2023asymptotic} and deep variant \citep{fabozzi2025asymptotic}; building confidence for DNN estimator \citep{meng2026inference}; estimating DNN with fat tail error distribution \cite{fan2024noise}. However, as far as we know, there is no direct literature to show the asymptotic distribution of the DNN estimator with dependent data. For the moment, without interpreting the main focus of this paper, we assume:
\begin{itemize}
    \item [A7.] Under the time series data $\{Y_1,\ldots,Y_n\}$ being $\beta$-mixing with at least exponentially fast mixing rate as mentioned in A2 and the DNN class being $\mathcal{N}_n$, fixing a specific hyperparameter setting, the distribution of $a_n A_{n+j}$ converges to a non-trivial continuous distribution for some appropriate sequence $a_n\to\infty$ as $n\to\infty$.
\end{itemize}
\begin{remark}
    More specifically, assumption A7 conveys the condition that the non-trivial limiting distribution of $a_n A_{n+j}$ with $\beta$-mixing time series only depends on the stationary distribution $\mathbb{P}_{\pi}$, the DNN class $\mathcal{N}_n$ and the hyperparameter setting. 
\end{remark}
In fact, assumption A7 is not rare in the literature. \cite{wolf2015bootstrap} assumed a similar condition in the derivation of their prediction interval. We should also note that this kind of assumption is common in the subsampling estimation area, in which it is the so-called minimal assumption required to make a confidence region for general parameter estimation \citep{politis1994large}; see more examples from \cite{politis1999subsampling,politis2001asymptotic,politis2024scalable}. However, to achieve the pertinence, we need that $a_n A_{n+j}^*$ converges to the same non-trivial continuous distribution as $a_n A_{n+j}$ in probability. We argue that this condition holds without any additional assumptions as long as the same DNN class $\mathcal{N}_n$ and hyperparameter setting are applied in the training phase with the bootstrap time series. The idea hinges on the property of forward bootstrap time series, which could sustain the $\beta$-mixing and possess the same stationary distribution as $\mathbb{P}_{\pi}$, inspired by the spirit of \cite{franke2002properties, franke2004bootstrapping}. This claim is summarized in the following \cref{Lemma:propertyofbootstrapseries}. Before that, we first present the sufficient conditions to guarantee the geometrically $\beta$-mixing property of the series generated by \cref{Eq:DGP}:
\begin{itemize}
    \item [E1.] $|f_0(\bm{y})| \leq \lambda \max \left\{\left|y_1\right|, \ldots,\left|y_d\right|\right\}+C$, for a positive number $\lambda<1$ and a constant C, where $\bm{y}:=(y_1,\ldots,y_d)$.
    \item [E2.] The probability density function of $\epsilon$, $p_\epsilon(\cdot)$ is continuous and everywhere positive.
\end{itemize}
The proof to show the geometric ergodicity, which is equivalent to $\beta$-mixing with at least an exponentially fast mixing rate \citep{bradley2005basic}, for the series generated by \cref{Eq:DGP} can be found in \cite{min1999probabilistic,an1996geometrical,wu2024bootstrap}. Compared to the assumptions list in these works, we actually need one more condition that the conditional mean function $f_0(\cdot)$ satisfies the inequality $\sup _{\|\bm{y}\|_2 \leq C}|f_0(\bm{y})|<\infty, \text { for each } C>0,$ where $\|\cdot\|_2$ is the Euclidean norm. However, this condition is met due to A2 in which we require $\|f_0\|_{\infty}\leq B_n$. When $d=1$, the condition E1 degenerates to Assumption 3 (i) in the work of \cite{franke2002properties}. 

The procedure to make PPI is presented in \cref{alg:PPI}, where $\hat{f}_n$ can be the DNN or Kernel estimator. In this paper, we attempt to verify the feasibility of PPI with the DNN estimator; see \cite{politis2023multi} for the discussion with kernel estimators. Step 11 of \cref{alg:PPI} is designed to ensure stationarity by taking the burn-in number as $m$; Step 12 is designed to ensure the predictions in the real and bootstrap worlds have the same conditional latest observation. In practice, we can generate residuals from $\hat\epsilon_t := Y_t - \hat{f}(\bm{Y}_{t-1})$ for $t = 1, \ldots, n$. However, in the proof of \cref{Lemma:propertyofbootstrapseries}, we need the density of $\hat\epsilon_t$ to be continuous and everywhere positive. To satisfy this requirement, we can apply a convolution or kernel estimation technique to acquire an everywhere-positive $\hat{p}_\epsilon$ in the bootstrap world to satisfy this condition. More specifically, we can consider $
\widetilde{\epsilon}_i=\hat{\epsilon}_i+z_i, \text { for } i=1, \ldots, n,$ where $z_i \sim N(0, \xi(n))$, where $\xi(n)$ converges to 0 as $n \rightarrow \infty$ with a suitable rate. It turns out that the probability density of $\widetilde{\epsilon}_i$ is just the kernel density estimation of $\hat{p}_\epsilon$ with Gaussian kernel, i.e., 
\begin{equation}\label{Eq:gaussianKernel}
    \hat{p}_{\epsilon}(x)=\frac{1}{n} \sum_{i=1}^n \frac{1}{\sqrt{2 \pi} h} \exp \left(-\frac{\left(x-\hat{\epsilon}_i\right)^2}{2 h^2}\right),
\end{equation}
where $h$ plays the same role as $ \xi(n)$. In theory, we will apply the kernel density estimation $\hat{p}(x)$, but it is free to generate pseudo residuals from the discrete centered $\{\hat{\epsilon}_i\}_{i=1}^n$.

\begin{algorithm}[htbp]
        \caption{Multi-step ahead PPI based on forward bootstrap}\label{alg:PPI}
        \KwData{Observations $(Y_{-d+1},\ldots,Y_0,Y_1, \ldots, Y_n)$.}
        \KwResult{$100(1-\alpha)\%$ $\text{PPI}_{n+j}(\alpha)$ for $Y_{n+j}$; $j = 1,\ldots, J$.}
        Estimate $f_0$ by $\hat{f}_n$ with DNN or Kernel estimators.\;
        \For{$t$ in $1, \ldots, n$}{
            Compute fitted value $\hat{Y}_t \leftarrow \hat{f}_n(\bm{Y}_{t-1})$ and fitted residual $\hat{\epsilon}_t \leftarrow Y_t - \hat{Y}_t$. \; 
        }
        Compute mean of fitted residuals $\overline{\hat{\epsilon}} \leftarrow \frac{1}{n}\sum_{t=1}^{n} \hat{\epsilon}_t$. \;
        Center the residuals $r_t \leftarrow \hat{\epsilon}_t - \overline{\hat{\epsilon}}$ and let $\widehat{F}_{cr}$ denote the empirical distribution of $r_t$ for $t = 1, \ldots, n$. \;
        \For{$b$ in $1, \ldots, B$}{
            Draw $\epsilon_{2}^*,\ldots, \epsilon_{m+n}^*  \overset{\text{i.i.d.}}{\sim} \widehat{F}_{cr}$. \;
            Choose $d$ continuous elements $\bm U_1^*$ at random from $Y_{1},\ldots, Y_{n}$. \;
            Generate $U_{2}^*,\ldots, U_{m+n}^*$ via 
                \begin{algomathdisplay}
                    \begin{aligned}
                        U_t^* \leftarrow \hat{f}_n(\bm U_{t-1}^*) + \epsilon_t^*;~\textbf{for~}t = 2, \ldots, m+n.
                    \end{aligned} 
                \end{algomathdisplay}
            Set $Y_t^* \leftarrow U_{t+m}^*$ for $t = 1, \ldots, n$. Re-estimate the model estimator in the bootstrap world to get $\hat{f}^*_n$. \;
            Set $\bm Y_n^* \leftarrow \bm Y_n$. \;
            Compute future point predictive values $\hat{Y}^*_{n+1},\ldots, \hat{Y}^*_{n+J}$ in the bootstrap world via \cref{alg:QPI} with $\hat{f}^*_n$. \;
            Compute future bootstrap observations $Y^*_{n+1},\ldots, Y^*_{n+J} $ via 
            \begin{algomathdisplay}
                \begin{aligned}
                    &Y_{t}^* \leftarrow \hat{f}_n(\bm Y_{t-1}^*) + \epsilon_{t}^*~\text{for}~t = n+1,\ldots, n+J~\text{by pretending}~\hat{f}_n~\text{is the true model,} \\
                    &\text{ where}~\epsilon_{t}^*~\text{are generated from}~\widehat{F}_{cr}.
                \end{aligned} 
            \end{algomathdisplay}
            Compute bootstrap roots replicate $R^{(b)*}_{n+j} = Y_{n+j}^* - \hat{Y}_{n+j}^*$; $j = 1,\ldots, J$.\;
        }
        Determine $\hat{q}^*_h(\alpha/2)$ and $\hat{q}^*_h(1-\alpha/2)$ which are the $\alpha/2$ and $1-\alpha/2$-quantile of the $h$-th step ahead predictive roots $\{R^{(b)*}_{n+j}\}_{b=1}^B$, respectively.
        
        Compute predicted future values $\hat{Y}_{n+1},\ldots, \hat{Y}_{n+j}$ in the real world via   \cref{alg:QPI} with $\hat{f}_n$; $j = 1,\ldots, J$.
        
        Construct the pertinent prediction interval for $Y_{n+j}$ as 
        \begin{algomathdisplay}
            \begin{aligned}
                \text{PPI}_{n+j}(\alpha) : [\hat{Y}_{n+j} + \hat{q}^*_h(\alpha/2),  \hat{Y}_{n+j} + \hat{q}^*_h(1-\alpha/2)]~\text{for}~ j = 1,\ldots, J.
            \end{aligned} 
        \end{algomathdisplay}
    \end{algorithm}

As we have discussed above, the crucial issue is to show that the limiting distribution of $a_n A_{n+j}^*$ is the same as $a_n A_{n+j}$. With \cref{Lemma:propertyofbootstrapseries} below and assumption A7, this fact can be verified. Since the assumption of the limiting distribution of $a_n A_{n+j}$ is unavoidable, we call our PPI with a DNN estimator built under minimal assumptions. 

\begin{lemma}\label{Lemma:propertyofbootstrapseries}
With all conditions assumed in \cref{Theorem:esterrorDNN} and C1 to C3, for the forward bootstrap time series $\{Y^*_t\}_{t=1}^n$, it is still $\beta$-mixing with at least exponentially fast mixing rate with different coefficients for $(Y_1,\ldots Y_n)\in \Omega_n$, where $\Omega_n := [-B_n,B_n]^n$. In addition, the stationary distribution of the forward bootstrap series $P^*_{\pi}$ converges to $P_{\pi}$ in probability. 
\end{lemma}

With \cref{Lemma:propertyofbootstrapseries} and A7, we conclude that the distribution of $a_n A_{n+j}^*$ converges to the same non-trivial limiting distribution as $a_n A_{n+j}$, which further implies the condition D2. The performance of QPI and PPI is verified with simulations and real-data studies in the following sections. 
\FloatBarrier
\section{Simulations}\label{Sec:Simulation}
In this section, we deploy simulation studies to check the performance of QPI and PPI with DNN and standard kernel estimators. 
To get the DNN estimation, we rely on the {\tt torch} package from Python. We apply the hyperparameters with epochs being 100, learning rate being 0.0002; batch size being 10; the burn-in number $m = 100$; $S = 5000$; $B = 1000$; optimizer being {\tt optim.Adam} from {\tt torch} package. The depth and width of DNN estimators are chosen by the theoretical results shown above. To get the local constant kernel estimation, we take {\tt KernelReg} function from {\tt statsmodels} package in Python. For the choice of bandwidth, we take the least-squares cross-validation method.

To check the effects of the curse of dimensionality on both methods, we mimic the ``high-dimension'' situation by setting the sample size to be small and the $d$ is relative large. We consider the following three models:
\begin{itemize}
    \item Model 1: $Y_{t} = \log(Y_{t-1}^2+1) + \epsilon_t$.
    \item Model 2: $Y_{t} = \sin(\sum_{i=1}^{5}Y_{t-i}) + \epsilon_t$.
    \item Model 3: $Y_{t} = \sin(\sum_{i=1}^{10}Y_{t-i}) + \epsilon_t$.
\end{itemize}
where $\{\epsilon_t\}$ are taken as $i.i.d.$ normal. We consider sample sizes $n = 50, 200, 500$. When $n$ is small and $d$ is relatively large, a ``high-dimension'' scenario is created. We evaluate different PIs by the CoVerage Rate (CVR) and Length (LEN) for each step ahead prediction horizon. In addition, we also consider the performance of $L_2$ point prediction by measuring the Mean Squared (MSE). More specifically, we define CVR, LEN and MSE for $j$-step ahead prediction as follows:
$$
\text{CVR}_j = \frac{1}{R}\sum_{i=1}^R\mathbbm{1}(Y_{i,n+j}\in [\hat{L}_{i,j},\hat{U}_{i,j}]);~\text{LEN}_j = \frac{1}{R}\sum_{i=1}^R(\hat{U}_{i,j} - \hat{L}_{i,j});~ \text{MSE}_j = \frac{1}{R}\sum_{i=1}^R(Y_{i,j} - \hat{Y}_{i,j})^2,
$$
where $R$ is the number of replications, which is taken as 1000 in our simulation studies; $Y_{i,n+j}$ is the true $j$-step ahead out-of-sample observations in $i$-th replication; $\hat{L}_{i,j}$ and $\hat{U}_{i,j}$ are lower and upper bounds of different $j$-step ahead prediction intervals in $i$-th replication, respectively; $\hat{Y}_{i,j}$ is the optimal $L_2$ point prediction from different methods. For the simulation, we consider $j = 1,\ldots, 5$. We should mention that our CVR and LEN measures are close to unconditional evaluation of different methods since the dependence on training datasets are canceled out by taking the average on multiple training datasets; see \cite{wang2025model} for discussions on various conditioning. All simulation results presented in \cref{tab:SM1,tab:SM2,tab:SM3}. To make our simulation studies reproducible, we set {\tt np.random.seed(i)} for $i = 1,\ldots, 1000$ to control the randomness to generate the data and {\tt torch.manual\_seed(1)} to control the initialization of the DNN parameters. All simulations are performed using computational services provided by the OSG Consortium \cite{osg07,osg09,osg2,osg1}. To accommodate the running time of the OSG Consortium, we use hyperparameter epochs 50; batch size 50, $S = 5000$ and $B = 500$ for Model 3 data with $n = 500$. 

From \cref{tab:SM1,tab:SM2,tab:SM3}, it is clear that the point prediction from the DNN estimator is more accurate than the variant with the kernel estimator, especially for Models 2 and 3 when the dimension $d$ is relatively large. For the prediction interval, a common phenomenon is that the interval length and coverage rate of PPI tend to be larger than those of the QPI, regardless of whether it is built based on DNN or kernel estimators. In addition, the PPI and QPI based on kernel estimators are struggling to achieve the nominal level, especially for Models 2 and 3 simulation data. For example, only 89.8$\%$ empirical coverage rate can be achieved by PPI with the kernel estimator for 1-step ahead prediction with 500 Model-3 simulation data. On the other hand, PPI with the DNN estimator can achieve 95.5$\%$. Besides, the PPI can render a better coverage rate compared to QPI, especially when the sample size is small. In practice, participants do not know whether the sample size is sufficiently large. Therefore, we suggest the PPI with the DNN estimators as the optimal approach. Also, we should remark that the running time of performing the PPI with the DNN estimator is resolvable since the bootstrap estimation can be done in parallel.    

\begin{table}[htbp]
    \centering
\begin{tabular}{lccccc|cccccc}
\toprule
Step & 1 & 2 & 3 & 4 & 5 & 1 & 2 & 3 & 4 & 5\\
\midrule
 & \multicolumn{5}{c}{DNN} &  \multicolumn{5}{c}{Kernel}\\[3pt]
Model-1; n = 50 & \\[3pt]
CVR-QPI & 0.927 & 0.898 & 0.901 & 0.884 & 0.889 & 0.886 & 0.906 & 0.923 & 0.902 & 0.906 \\
LEN-QPI & 4.105 & 4.352 & 4.353 & 4.355 & 4.354 & 3.533 & 4.385 & 4.552 & 4.607 & 4.635 \\
CVR-PPI & 0.941 & 0.902 & 0.904 & 0.893 & 0.897 & 0.915 & 0.913 & 0.921 & 0.910 & 0.913 \\
LEN-PPI & 4.276 & 4.376 & 4.377 & 4.376 & 4.381 & 4.068 & 4.552 & 4.669 & 4.717 & 4.739 \\
MSE & 1.259 & 1.604 & 1.640 & 1.795 & 1.751 & 1.203 & 1.573 & 1.593 & 1.794 & 1.739 \\[5pt]

Model-1; n = 200 & \\[3pt]
CVR-QPI & 0.935 & 0.943 & 0.955 & 0.934 & 0.945 & 0.931 & 0.945 & 0.952 & 0.936 & 0.944 \\
LEN-QPI & 3.843 & 4.524 & 4.712 & 4.780 & 4.810 & 3.817 & 4.536 & 4.714 & 4.773 & 4.798 \\
CVR-PPI & 0.932 & 0.944 & 0.955 & 0.931 & 0.942 & 0.940 & 0.945 & 0.959 & 0.932 & 0.936 \\
LEN-PPI & 3.875 & 4.542 & 4.740 & 4.802 & 4.825 & 3.955 & 4.577 & 4.746 & 4.793 & 4.815 \\
MSE & 1.066 & 1.339 & 1.421 & 1.674 & 1.677 & 1.101 & 1.343 & 1.420 & 1.667 & 1.677 \\[5pt]

Model-1; n = 500 & \\[3pt]
CVR-QPI & 0.950 & 0.940 & 0.944 & 0.941 & 0.954 & 0.956 & 0.945 & 0.941 & 0.937 & 0.953 \\
LEN-QPI & 3.875 & 4.594 & 4.776 & 4.840 & 4.864 & 3.870 & 4.585 & 4.767 & 4.828 & 4.847 \\
CVR-PPI & 0.953 & 0.938 & 0.946 & 0.940 & 0.952 & 0.955 & 0.943 & 0.941 & 0.935 & 0.952 \\
LEN-PPI & 3.882 & 4.603 & 4.786 & 4.841 & 4.866 & 3.935 & 4.600 & 4.768 & 4.825 & 4.842 \\
MSE & 0.942 & 1.474 & 1.545 & 1.555 & 1.556 & 0.964 & 1.460 & 1.522 & 1.535 & 1.536 \\
\bottomrule
\end{tabular}
    \caption{Simulation results with Model-1 data.}
    \label{tab:SM1}
\end{table}

\begin{table}[htbp]
    \centering
\begin{tabular}{lccccc|cccccc}
\toprule
Step & 1 & 2 & 3 & 4 & 5 & 1 & 2 & 3 & 4 & 5\\
\midrule
 & \multicolumn{5}{c}{DNN} &  \multicolumn{5}{c}{Kernel}\\[3pt]
Model-2; n = 50 & \\[3pt]
CVR-QPI & 0.907 & 0.920 & 0.938 & 0.929 & 0.931 & 0.546 & 0.595 & 0.681 & 0.724 & 0.754 \\
LEN-QPI & 4.417 & 4.489 & 4.602 & 4.680 & 4.707 & 2.453 & 2.841 & 3.167 & 3.431 & 3.635 \\
CVR-PPI & 0.920 & 0.922 & 0.943 & 0.931 & 0.938 & 0.693 & 0.712 & 0.750 & 0.766 & 0.802 \\
LEN-PPI & 4.634 & 4.655 & 4.703 & 4.764 & 4.800 & 4.176 & 4.202 & 4.243 & 4.345 & 4.358 \\
MSE & 1.639 & 1.615 & 1.498 & 1.657 & 1.542 & 2.352 & 2.180 & 1.985 & 2.133 & 1.888 \\[5pt]

Model-2; n = 200 & \\[3pt]
CVR-QPI & 0.877 & 0.894 & 0.922 & 0.911 & 0.914 & 0.755 & 0.835 & 0.877 & 0.908 & 0.928 \\
LEN-QPI & 4.051 & 4.185 & 4.292 & 4.335 & 4.397 & 3.457 & 3.942 & 4.269 & 4.468 & 4.638 \\
CVR-PPI & 0.895 & 0.912 & 0.931 & 0.914 & 0.916 & 0.826 & 0.861 & 0.898 & 0.907 & 0.928 \\
LEN-PPI & 4.359 & 4.405 & 4.436 & 4.439 & 4.451 & 4.705 & 4.746 & 4.767 & 4.775 & 4.811 \\
MSE & 1.709 & 1.576 & 1.636 & 1.619 & 1.519 & 2.002 & 1.777 & 1.821 & 1.722 & 1.553 \\[5pt]

Model-2; n = 500 & \\[3pt]
CVR-QPI & 0.880 & 0.879 & 0.899 & 0.916 & 0.931 & 0.801 & 0.873 & 0.913 & 0.927 & 0.936 \\
LEN-QPI & 3.747 & 4.067 & 4.244 & 4.354 & 4.431 & 3.391 & 4.032 & 4.440 & 4.617 & 4.747 \\
CVR-PPI & 0.910 & 0.906 & 0.911 & 0.922 & 0.931 & 0.844 & 0.891 & 0.912 & 0.926 & 0.937 \\
LEN-PPI & 4.176 & 4.375 & 4.451 & 4.493 & 4.520 & 4.639 & 4.766 & 4.857 & 4.835 & 4.862 \\
MSE & 1.494 & 1.629 & 1.600 & 1.564 & 1.503 & 1.732 & 1.717 & 1.680 & 1.598 & 1.520 \\
\bottomrule
\end{tabular}
    \caption{Simulation results with Model-2 data.}
    \label{tab:SM2}
\end{table}

\begin{table}[htbp]
    \centering
\begin{tabular}{lccccc|cccccc}
\toprule
Step & 1 & 2 & 3 & 4 & 5 & 1 & 2 & 3 & 4 & 5\\
\midrule
 & \multicolumn{5}{c}{DNN} &  \multicolumn{5}{c}{Kernel}\\[3pt]
Model-3; n = 50 & \\[3pt]
CVR-QPI & 0.897 & 0.912 & 0.918 & 0.894 & 0.933 & 0.021 & 0.018 & 0.022 & 0.035 & 0.034 \\
LEN-QPI & 4.343 & 4.407 & 4.441 & 4.483 & 4.534 & 0.078 & 0.104 & 0.116 & 0.146 & 0.169 \\
CVR-PPI & 0.918 & 0.920 & 0.934 & 0.899 & 0.936 & 0.150 & 0.143 & 0.157 & 0.183 & 0.163 \\
LEN-PPI & 4.600 & 4.601 & 4.589 & 4.603 & 4.643 & 1.160 & 1.211 & 1.208 & 1.257 & 1.330 \\
MSE & 1.641 & 1.655 & 1.479 & 1.706 & 1.494 & 3.004 & 2.919 & 2.896 & 2.957 & 2.839 \\[5pt]

Model-3; n = 200 & \\[3pt]
CVR-QPI & 0.848 & 0.859 & 0.879 & 0.869 & 0.879 & 0.521 & 0.582 & 0.628 & 0.630 & 0.651 \\
LEN-QPI & 3.950 & 4.008 & 4.077 & 4.118 & 4.162 & 2.392 & 2.599 & 2.798 & 2.960 & 3.129 \\
CVR-PPI & 0.866 & 0.896 & 0.913 & 0.892 & 0.896 & 0.639 & 0.665 & 0.731 & 0.709 & 0.729 \\
LEN-PPI & 4.348 & 4.365 & 4.396 & 4.383 & 4.393 & 3.981 & 3.983 & 4.008 & 4.041 & 4.085 \\
MSE & 1.813 & 1.721 & 1.677 & 1.804 & 1.715 & 2.315 & 2.308 & 2.040 & 2.307 & 2.075 \\[5pt]

Model-3; n = 500 & \\[3pt]
CVR-QPI & 0.961 & 0.933 & 0.934 & 0.943 & 0.954 & 0.869 & 0.880 & 0.896 & 0.902 & 0.942 \\
LEN-QPI & 4.633 & 4.641 & 4.650 & 4.656 & 4.667 & 4.054 & 4.277 & 4.407 & 4.482 & 4.579 \\
CVR-PPI & 0.955 & 0.931 & 0.933 & 0.937 & 0.945 & 0.898 & 0.881 & 0.889 & 0.902 & 0.933 \\
LEN-PPI & 4.645 & 4.651 & 4.652 & 4.660 & 4.649 & 5.020 & 4.972 & 4.894 & 4.851 & 4.835 \\
MSE & 1.359 & 1.533 & 1.641 & 1.472 & 1.438 & 1.686 & 1.752 & 1.819 & 1.683 & 1.545 \\
\bottomrule
\end{tabular}
    \caption{Simulation results with Model-3 data.}
    \label{tab:SM3}
\end{table}

\FloatBarrier

\section{Real Data Analysis}\label{Sec:realdata}
In practice, we are unable to replicate the simulation multiple times with the true data-generating model. To evaluate the performance of QPI and PPI with DNN and standard kernel estimators, we consider the moving window prediction procedure. In particular, if we have time series data $y_1,\ldots,y_n$, we will take $y_{1}, ...y_w$ to be the training window, where $w$ is a constant chosen by users. 
Then, we can construct the prediction intervals for the true observed $y_{w+1}, ...y_{w+J}$. Subsequently, we shift the window to $y_{2}, ...y_{w+1}$ and train different models and then make $j$-step ahead prediction intervals. Here, we consider $j = 1,\ldots,J$ and repeat the prediction procedure until we go through the last $J$-step ahead prediction. Lastly, the average coverage rate for each $j$-step ahead prediction over the $n-J-w+1$ number of moving window predictions can be used to measure the coverage performance of different intervals. In short, we compute the following two metrics for various prediction intervals:
\begin{equation*}
    \text{CVR}_j = \frac{1}{n-J-w+1}\sum_{i=w}^{n-J}\mathbbm{1}(y_{i+j}\in [\hat{L}_{i+j},\hat{U}_{i+j}]);~\text{LEN}_j = \frac{1}{n-J-w+1}\sum_{i=w}^{n-J}(\hat{U}_{i+j} - \hat{L}_{i+j}),
\end{equation*}
where $y_{i+j}$ is the ($i+j$)-th observed data; $\hat{L}_{i+j}$ and $\hat{U}_{i+j}$ are the corresponding lower and upper bounds of various prediction intervals. More specifically, $\text{CVR}_j$ and $\text{LEN}_j$ depict the average coverage rate and the average interval length of different methods on $j$-step ahead moving window predictions.

We consider the M4 competition dataset \citep{makridakis2020m4}. To make sure we have a good evaluation of the performance of different intervals, the value of $n-J-w+1$ needs to be sufficiently large. Thus, we consider yearly data series which have a length between 300 and 350, so $n-J-w+1$ is at least 250 if $w = 50$, resulting in 7 different series in total. We should mention that the preprocessing of each time series can be very subtle; see \cite{smyl2020hybrid} for some practical suggestions. In this work, we do not focus on the strategy of preprocessing time series, so we take simple first-order differencing on all series.

To compare the performance of PI based on the DNN and the kernel estimators, we consider $w = 50$ and $w = 100$. Meanwhile, we consider the input dimension $d = 1,5,10$. In practice, the order selection is a crucial but non-trivial problem in the modeling stage of time series prediction. In this paper, our focus is on the pertinent prediction with DNN estimators and its comparison with the predictions relying on the classical kernel estimators. Thus, we skip the order selection step and consider multiple order values to evaluate various prediction methods. We take $J = 5$ and we first compute the overall average coverage rate (OCVR) and the overall average prediction interval length (OLEN) over all 7 series and $J$ step-ahead predictions. More specifically, we compute the average coverage rate and interval length over each prediction horizon for each time series, i.e., $\frac{1}{J}\sum_{j=1}^J\text{CVR}_j$ and $\frac{1}{J}\sum_{j=1}^J\text{LEN}_j$. We further consider the average value of $\frac{1}{J}\sum_{j=1}^J\text{CVR}_j$ and $\frac{1}{J}\sum_{j=1}^J\text{LEN}_j$ on each time series and call the final two metrics the OCVR and OLEN. This result is presented in \cref{tab:realOverall}. Beyond the overall average results, we also report the overall sample standard deviation of the coverage rate and the interval length, namely StdCVR and StdLEN, in \cref{tab:realOverallstd}. From these two tables, it is clear that the performance of the PI based on the kernel estimator spoils when the ratio $w/d$ is small. On the other hand, the performance of the PI based on the DNN estimator, especially the PPI with the DNN estimator, is more stable. Although the overall coverage rate of PI with the kernel estimator is slightly larger than that of the variant with the DNN estimator when the ratio $w/d$ is large (e.g., $d = 1$), the overall average interval length and its corresponding sample standard deviation are larger when the kernel estimator is considered, which stands for a drawback in practice. More importantly, people do not know which order $d$ should be used with real data. Since the PI with the DNN estimator is more stable with varying window size and model order, it shall be a superior option for real studies compared to the PI with the kernel estimator. In particular, we will suggest the PPI with DNN estimator, which is also the conclusion made based on the simulation results.

\begin{table}[]
    \centering
    \begin{tabular}{cccc|ccc}
\toprule
   & \multicolumn{3}{c}{DNN} &  \multicolumn{3}{c}{Kernel}\\[3pt]
\midrule
& \multicolumn{6}{c}{$w = 50$}\\[3pt]
d: & 1 & 5 & 10 &  1 & 5 & 10 \\
     OCVR-QPI  & 0.904 & 0.839 & 0.765  & 0.901 & 0.667 & 0.140\\
     OLEN-QPI  & 249.721 & 177.748 & 253.878 & 203.281 & 167.785 & 174.764 \\
     OCVR-PPI  & 0.911 & 0.862 & 0.803 & 0.917 & 0.748 & 0.182 \\
     OLEN-PPI  & 302.778 & 269.948 & 265.031 & 320.283 & 296.002 & 248.926\\[5pt]
& \multicolumn{6}{c}{$w = 100$}\\[3pt]
    OCVR-QPI    & 0.889 & 0.776 & 0.510 & 0.895 & 0.795 & 0.290 \\
     OLEN-QPI    & 210.767 & 207.736 & 183.530 & 247.991 & 156.444 & 115.583  \\
     OCVR-PPI    & 0.891 & 0.823 & 0.624  & 0.903 & 0.847 & 0.389\\
     OLEN-PPI & 245.699 & 259.832 & 239.462 & 274.770 & 215.613 & 199.710\\
\bottomrule
    \end{tabular}
    \caption{The average coverage rate and interval length of different methods over each prediction horizon $j$ and 7 selected time series, i.e., OCVR and OLEN defined in the main text.}
    \label{tab:realOverall}
\end{table}

\begin{table}[]
    \centering
    \begin{tabular}{cccc|ccc}
\toprule
   & \multicolumn{3}{c}{DNN} &  \multicolumn{3}{c}{Kernel}\\[3pt]
\midrule
& \multicolumn{6}{c}{$w = 50$}\\[3pt]
d: & 1 & 5 & 10 &  1 & 5 & 10 \\
     StdCVR-QPI  & 0.014 & 0.021 & 0.069  & 0.013 & 0.075 & 0.133\\
     StdLEN-QPI  & 309.195 & 245.363 & 444.964 & 285.897 & 275.712 & 385.182\\
     StdCVR-PPI  & 0.014 & 0.019 & 0.059  & 0.017 & 0.061 & 0.159 \\
     StdLENPPI  & 402.349 & 406.473 & 436.904 & 448.045 & 467.784 & 521.170\\[5pt]
& \multicolumn{6}{c}{$w = 100$}\\[3pt]
     StdCVR-QPI    & 0.021 & 0.056 & 0.191 & 0.016 & 0.070 & 0.160 \\
     StdLEN-QPI    & 281.192 & 339.462 & 358.694 & 387.294 & 242.921 & 259.348 \\
     StdCVR-PPI    & 0.019 & 0.031 & 0.151 & 0.017 & 0.035 & 0.186 \\
     StdLEN-PPI & 286.280 & 391.971 & 435.777 & 435.848 & 310.575 & 404.207\\
\bottomrule
    \end{tabular}
    \caption{The sample standard deviation of the coverage rate and interval length from different methods over each prediction horizon $j$ and 7 selected time series, i.e., the spread of the overall average coverage rate and the interval length.}
    \label{tab:realOverallstd}
\end{table}

\FloatBarrier
\section{Discussion}\label{Sec:conclusion}
In this work, we consider the prediction inference of one type of time series data generated by an additive model with the DNN estimator. We show that the point predictions are consistent and the corresponding PIs are asymptotically valid under standard assumptions. To improve the coverage rate of the PI in the finite sample cases, we further give the algorithm to build the so-called PPI. Under the minimal assumptions, we verify that the PPI is still attainable. Moreover, as revealed by the simulation and real data studies, the PI with a DNN estimator is more robust against the curse of dimensionality, compared to the variant with the kernel estimator. Also, the PPI is shown to possess a better coverage rate than the naive QPI. 

There are several possible directions in which to extend the current work. First, the order selection for prediction with kernel and DNN estimators is not trivial. A practical, useful order selection criterion is important for participants. Secondly, the prediction inference of the general model, i.e., without the additive model structure constraint, has not been explored so far. Thirdly, although the training of DNN on pseudo data can be parallelized, the computation is still heavy in terms of memory. A more computationally efficient way to build PPI is desired in future work. 
 
\section*{Acknowledgement}
This research was done using computational services provided by the OSG Consortium \cite{osg07,osg09,osg2,osg1}, which is supported by the National Science Foundation awards 2030508 and 1836650. 

\section*{Declaration of Interest Statement}
The author declares no conflict of interest.

\section*{Declaration of Generative AI Use}
The author reports generative AI was not used in their research or preparation of this manuscript.

\clearpage
\bibliographystyle{apalike}
\bibliography{refs}

\appendix

\section{The Illustration of DNN Structure}\label{Appendxi:IllustrationStructure}
We present the structure of a simple fully connected DNN in \cref{FigDNNstructureex} for reference.
\begin{figure}[htbp]
 \centering
 \caption{The illustration of a fully connected DNN with $L = 2$, $H = 4$ and $W = 37$, and input dimension $d = 2$ and output dimension 1.}
 \includegraphics[scale = 0.5]{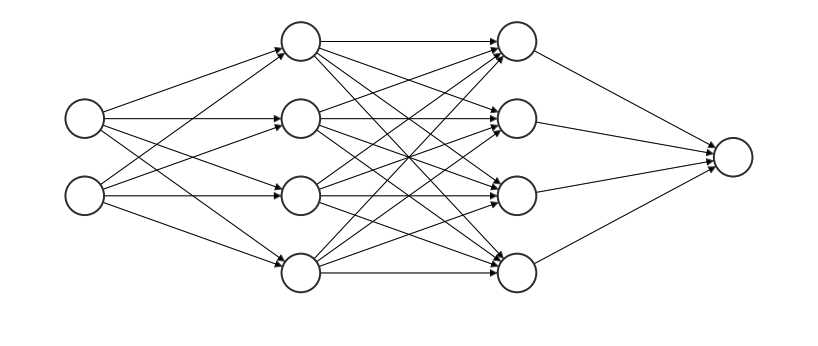}
 \label{FigDNNstructureex}
\end{figure}

\section{Proof}\label{Appendxi:proof}
\subsection{Proof of \cref{Lemma:theEstimationability}}

\begin{proof}
We first consider the function class $\mathcal{F}^{s}_{r,d, [0,1]^d}$ in a more constrained domain, i.e. for $\forall~g\in \mathcal{F}^{s}_{r,d, [0,1]^d}$, s.t.,
$$
\|g\|_{C^{k, \alpha}\left([0,1]^d\right)}=\max \left\{\max _{\bm{k}:|\bm{k}| \leq k} \max _{\bm{u} \in[0,1]^d}\left|D^{\bm{k}} g(\bm{u})\right|, \max _{\bm{k}:|\bm{k}|=k} \sup _{\substack{\bm{u}, \bm{v} \in[0,1]^d \\ \bm{u} \neq \bm{v}}} \frac{\left|D^{\bm{k}} g(\bm{u})-D^{\bm{k}} g(\bm{v})\right|}{\|\bm{u}-\bm{v}\|^\alpha}\right\}~;~\|g\|_\infty\leq B_n.
$$
Then, we will extend the domain $[0,1]^d$ to $[-B_n, B_n]^d$. 

From Theorem 3.1 of \cite{yarotsky2020phase}, there exists a general DNN $\tilde{f}$ with depth $\tilde{L}_n\leq C_1 \log \tilde{W}_n$, in which the neurons may not be fully connected, s.t., $\|\tilde{f} - g_0 \|_{\infty}\leq C_2\tilde{W}_n^{-r/d}$; where $g_0 \in \mathcal{F}^{s}_{r,d, [0,1]^d}$; $C_1$ and $C_2$ are two constants independent with $\tilde{W}_n$ which is the number of total parameters in $\tilde{f}$. Then, by Lemma 1 of \cite{farrell2021deep}, this general DNN $\tilde{f}$ can be embedded into a fully connected equal width DNN with width $H_n\leq C_3 \tilde{W}_n \tilde{L}_n$ and the depth $L_n = \tilde{L}_n$; where $C_3$ is a constant independent with $\tilde{W}_n$ and $\tilde{L}_n$. Thus, $\exists f_{DNN} \in \mathcal{N}(L_n,H_n,2B_n)$, s.t., $f_{DNN} = \tilde{f}$. 

We should note that $\tilde{f}\in\mathcal{N}(L_n,H_n,2B_n)$ is feasible if $W_n$ is large s.t., $C_2\tilde{W}_n^{-r/d}\leq B_n$. Then, we have $\|\tilde{f}\|_{\infty} \leq \|g_0\|_{\infty} +  C_2\tilde{W}_n^{-r/d} \leq 2B_n$. Then, we can take $\tilde{W}_n^{-r/d} = n^{-\eta\frac{r}{d}}$; where $\eta$ is an appropriate constant belongs $(0,1)$. This implies that $n = \tilde{W}_n^{\frac{1}{\eta}}$ which means the sample size is larger that the number of total parameters. 

    Summarizing the above discussions, when $C_2 n^{-\eta\frac{r}{d}} \leq B_n$ for sufficiently large $n$, we have an $f_{DNN} \in \mathcal{N}(L_n,H_n,2B_n)$, where
    $$
    L_n = O(\log n); H_n = O(n^{\eta}\log n)
    $$
    s.t.,
    $$
    \|f_{DNN} - f_0 \|_{\infty} \leq n^{-\eta\frac{r}{d}}.
    $$

    Then, we attempt to extend the results to $f_0$ which belongs to a larger function class $\mathcal{F}_{r,[-B_n,B_n]^d} := \{f:  \|f\|_{C^{k, \alpha}\left([-B_n,B_n]^d\right)}< \infty~\text{and}~\|f\|_\infty\leq B_n\}$. We define $f(\bm{x}) = g(T^{-1}(\bm{x}))$; where $\bm{x} \in [-B_n,B_n]^d$ and $T(\bm{u}) = 2B_n\bm{u} - B_n$. Thus, $T(\bm{u}): [0,1]^d \to [-B_n,B_n]^d$. Let $T^{-1}(\bm{x}) = \bm{u}$. $g(T^{-1}(\bm{x})) = g(\bm{u})$ is a function on domain $[0,1]^d$. Moreover, we have $g(\bm{u}) = f(T(\bm{u}))$. By the chain rule and $T^{-1}(\bm{x}) = \frac{\bm{x}+B_n}{2B_n}$, we have 
    $$
    \left|D^{\bm{k}} g(\bm{u})\right| = (2B_n)^{\bm{k}} \left|D^{\bm{k}} f(\bm{x})\right|; \frac{\left|D^{\bm{k}} g(\bm{u})-D^{\bm{k}} g(\bm{v})\right|}{\|\bm{u}-\bm{v}\|^\alpha} = \frac{(2B_n)^r\left|D^{\bm{k}} f(\bm{x})-D^{\bm{k}} f(\bm{y})\right|}{\|\bm{x}-\bm{y}\|^\alpha}.
    $$

Then, consider the function $\frac{g(\bm{u})}{(2B_n)^r}$, it is clear that $\|\frac{g(\bm{u})}{(2B_n)^r}\|_{C^{k, \alpha}\left([0,1]^d\right)}< \infty$ by assumption A1 on the finiteness of function class of $f$. Therefore, $\exists g^*_{DNN}\in \mathcal{N}(L_n,H_n,2B_n)$ s.t., $\|g^*_{DNN} - \frac{g(\bm{u})}{(2B_n)^r} \|_{\infty} \leq n^{-\eta\frac{r}{d}}$. Therefore, $(2B_n)^r\|g^*_{DNN} - \frac{g(\bm{u})}{(2B_n)^r} \|_{\infty} \leq (2B_n)^rn^{-\eta\frac{r}{d}}$. Denote $f^*_{DNN} =(2B_n)^r g^*_{DNN}$, we have $\|f^*_{DNN}(\bm{u}) - g(\bm{u}) \|_{\infty} \leq (2B_n)^rn^{-\eta\frac{r}{d}}$; where $f^*_{DNN}$ and $g^*_{DNN}$ have the same structure. Further notice that 
$$ \|f_1-f_2\|_{\infty} = \sup _{\bm{x} \in\left[-B_n, B_n\right]^d}|f_1(\bm{x})-f_2(\bm{x})|  =\sup _{\bm{x}}\left|g_1\left(T^{-1}(\bm{x})\right)-g_2\left(T^{-1}(\bm{x})\right)\right| = \sup _{\bm{u}}\left|g_1(\bm{u})-g_2(\bm{u})\right|;$$
for any two pairs of functions $f_1$ and $f_2$, $g_1$ and $g_2$. We can conclude that $\forall f_0 \in \mathcal{F}_{r,[-B_n,B_n]^d}$, $\exists f^*_{DNN} \in \mathcal{N}(L_n,H_n,(2B_n)^{r+1})$, s.t., $\|f^*_{DNN}(\bm{x}) - f_0(\bm{x}) \|_{\infty} \leq (2B_n)^rn^{-\eta\frac{r}{d}}$, where $T^{-1}(\bm u)$ is absorbed into the first layer of the DNN estimator $f^*_{DNN}$.
    \end{proof}

\subsection{Proof of \cref{Lemma:variantofTheorem2}}
\begin{proof}
The proof will be almost the same as the proof of Theorem 2 in \cite{brown2024statistical}. Compared to their results, the difference is that $\|f_0\|_{\infty}$ in our paper is assumed to be bounded by some nonecreasing sequence $\left\{D_n\right\}_{n \in \mathbb{N}}$ instead of being bounded by 1 in their assumptions. The proof is still relied on the localization analysis initially utilized by \cite{farrell2021deep} and the independent blocks applied by \cite{chen1998sieve}. In particular, the new assumption $\|f_0\|_{\infty}\leq D_n$ does not influence the proof details in Appendix D.2, D.5.2, D.5.3, D.5.4 of \cite{brown2024statistical} when their assumption $\|f_0\|_{\infty}\leq 1$ is applied. Therefore, we can have the same error bound for any sieve estimator and the corresponding function class satisfy assumptions B1 - B4. 
\end{proof}

\subsection{Proof of \cref{Theorem:esterrorDNN}}

\begin{proof}
    To prove this theorem, we can rely on \cref{Lemma:variantofTheorem2} and check if B1 to B4 therein are satisfied by assuming A1 to A3 with DNN estimators from $\mathcal{N}_n$.

    For assumption B1, by \cref{Lemma:theEstimationability} and the setting $L_n \asymp \log(n)$ and $H_n \asymp n^{\left(\frac{d}{r+d}\right)\left(1 / 2-K_B\right)} \log (n)$ in \cref{Theorem:esterrorDNN}, we have 
    $$
    \left\|f_{D N N}^*-f_0\right\|_{\infty} \leq\left(2 B_n\right)^r n^{-\left(\frac{d}{r+d}\right)\left(1 / 2-K_B\right) \frac{r}{d}} = \left(2 B_n\right)^r n^{-\left(\frac{r}{r+d}\right)\left(1 / 2-K_B\right)}.
    $$
    In particular, $\left(2 B_n\right)^r n^{-\left(\frac{r}{r+d}\right)\left(1 / 2-K_B\right)} = O(n^{rK_B-\left(\frac{r}{r+d}\right)\left(1 / 2-K_B\right)})$. By the range of $K_B$, $K_B < \frac{1}{2(r+d)+2}$, we have $rK_B-\left(\frac{r}{r+d}\right)\left(1 / 2-K_B\right)<0$. Thus, $B2$ is satisfied as $n\to\infty$.

    For B2, since $\sup_{f\in\mathcal{N}_n}\|f\|_{\infty}\leq 2B_n <\infty$, it is a routine check, which has been done in \cite{farrell2021deep,brown2024statistical}, etc. 

    For B3, in a similar proof of Lemma 14 in \cite{brown2024statistical}, we have $\operatorname{Pdim}\left(\mathcal{F}_n\right) \asymp n^{2\eta}\log^5(n)$ for $L_n = O(\log n); H_n = O(n^{\eta}\log n)$. For the specific setting in \cref{Theorem:esterrorDNN}, we have $\operatorname{Pdim}\left(\mathcal{F}_n\right) \asymp n^{2\left(\frac{d}{r+d}\right)\left(1 / 2-K_B\right)}\log^5(n)$. Thus,
    $$
    \frac{B_n}{\sqrt{n}}\left[\sqrt{Pdim\left(\mathcal{F}_n\right) \log (n)}+\sqrt{\log \log (n)}\right] = O\left( n^{\left(\frac{d}{r+d}\right)\left(1 / 2-K_B\right) - 1/2 + K_B}\log^3(n) \right) \to 0,
    $$
    since $d/(d+p) < 1$.

    For B4, taking $m_n(\bm{Z}_t):=2\sup_{f\in\mathcal{N}_n}|Y_t  - f(\bm{Y}_{t-1})|$, the check of the first part $|q(\boldsymbol{z}, f(\boldsymbol{z}))$ $-q\left(\boldsymbol{z}, f^{\prime}(\boldsymbol{z})\right)| \leq m_n(\boldsymbol{z})\left|f(\boldsymbol{z})-f^{\prime}(\boldsymbol{z})\right|$ is again a routine practice with $q(\boldsymbol{z}, f(\boldsymbol{z}))$ being least square loss. 
    
    For the second part, we can take $\mu_n:=\max \left\{6 \mathbb{E}\left[Y_t^2 \mathbbm{1}_{\left\{\left|Y_t\right| \geq B_n\right\}}\right], n^{-1}\right\}$. For $\forall~f\in \left\{\mathcal{F}_n \cup\left\{f_0\right\}\right\}$ and taking $C_5 = 8$, we first consider
    \begin{equation}\label{Eq:mainB4}
        \begin{split}
        &\mathbb{E}\left[\left|q\left(\boldsymbol{Z}_t, f\left(\boldsymbol{Z}_t\right)\right)\right| \mathbbm{1}_{\left\{m_n\left(\boldsymbol{Z}_t\right) \geq 8 B_n\right\}}\right] \\
        &\leq \mathbb{E}\left[\left|Y_t^2+f\left(\bm{Y}_{t-1}\right)^2-2 f\left(\bm{Y}_{t-1}\right) Y_t\right| \mathbbm{1}_{\left\{\left|Y_t\right| \geq 2B_n\right\}}\right] \\ & \leq \mathbb{E}\left[Y_t^2 \mathbbm{1}_{\left\{\left|Y_t\right| \geq 2B_n\right\}}\right]+\mathbb{E}\left[f\left(\bm{Y}_{t-1}\right)^2 \mathbbm{1}_{\left\{\left|Y_t\right| \geq 2B_n\right\}}\right]+2 \mathbb{E}\left[\left|f\left(\bm{Y}_{t-1}\right) Y_t\right| \mathbbm{1}_{\left\{\left|Y_t\right| \geq 2B_n\right\}}\right] \\ & \leq \mathbb{E}\left[Y_t^2 \mathbbm{1}_{\left\{\left|Y_t\right| \geq 2B_n\right\}}\right]+4B_n^2 P\left(\left|Y_t\right| \geq 2B_n\right)+4 B_n \mathbb{E}\left[\left|Y_t\right| \mathbbm{1}_{\left\{\left|Y_t\right| \geq 2B_n\right\}}\right],
        \end{split}
    \end{equation}
    where the first inequality comes from the expansion of $q\left(\boldsymbol{Z}_t, f\left(\boldsymbol{Z}_t\right)\right)$ and the fact $\mathbbm{1}_{\left\{m_n\left(\boldsymbol{Z}_t\right) \geq 8B_n\right\}} \leq \mathbbm{1}_{\{Y_t\geq B_n\}}$ for all $n$ and $t$. To see this, we observe that $m_n(\bm{Z}_t) \leq 2(|Y_t| + 2B_n)$ for all $n$ and $t$ by the definition of $\mathcal{N}_n$. Then, 
    \begin{equation}\label{Eq:maxMY}
        \begin{split}
            \left\{\omega: \max _{t \in\{1, \ldots, n\}} m_n\left(\boldsymbol{Z}_t(\omega)\right) \geq 8B_n\right\} & \subseteq\left\{\omega: \max _{t \in\{1, \ldots, n\}} 2\left(\left|Y_t(\omega)\right|+2B_n\right) \geq 8B_n\right\}  \\ & =\left\{\omega: \max _{t \in\{1, \ldots, n\}}\left|Y_t(\omega)\right| \geq 2B_n\right\};
        \end{split}
    \end{equation}
    the second inequality comes from the triangle inequality; the last inequality comes from the fact $\sup_{f\in\mathcal{N}_n}\|f\|_{\infty}\leq 2B_n$. Moreover, by Markov inequality, we have 
    \begin{equation*}
        \begin{split}
      &4B_n^2 P\left(\left|Y_t\right| \geq 2B_n\right)\\
      &=4B_n^2 P\left(\left|Y_t\right| \mathbbm{1}_{\left\{\left|Y_t\right| \geq 2B_n\right\}} \geq 2B_n\right)\\
      &=4B_n^2 P\left(Y_t^2 \mathbbm{1}_{\left\{\left|Y_t\right| \geq 2B_n\right\}} \geq4 B_n^2\right) \\
      & \leq \mathbb{E}\left[Y_t^2 \mathbbm{1}_{\left\{\left|Y_t\right| \geq 2B_n\right\}}\right],      
        \end{split}
    \end{equation*}

    and we have 
    $$
    B_n \mathbb{E}\left[\left|Y_t\right| \mathbbm{1}_{\left\{\left|Y_t\right| \geq 2B_n\right\}}\right] \leq \mathbb{E}\left[Y_t^2 \mathbbm{1}_{\left\{\left|Y_t\right| \geq 2B_n\right\}}\right].
    $$
    Thus, by stationarity, \cref{Eq:mainB4} implies that 
    $$
    \sup _{f \in\left\{\mathcal{F}_n \cup\left\{f_0\right\}\right\}}\mathbb{E}\left[\left|q\left(\boldsymbol{Z}_t, f\left(\boldsymbol{Z}_t\right)\right)\right| \mathbbm{1}_{\left\{m_n\left(\boldsymbol{Z}_t\right) \geq 8 B_n\right\}}\right] \leq 6 \mathbb{E}\left[Y_t^2 \mathbbm{1}_{\left\{\left|Y_t\right| \geq 2B_n\right\}}\right].
    $$
    With assumption A3, by Lemma 1 of \cite{brown2024statistical}, we have $\mathbb{E}\left[Y_t^2 \mathbbm{1}_{\left\{\left|Y_t\right| \geq 2B_n\right\}}\right]\to 0 $ as $n\to\infty$. The four assumptions in \cref{Lemma:variantofTheorem2} are satisfied in the context of \cref{Theorem:esterrorDNN}.

    To finish the proof, we need to check remaining conditions in \cref{Lemma:variantofTheorem2}, i.e., we need to verify the existence of $a$ and $\delta$. We take 
    $$
    \delta = C_\delta n^{2rK_B+2\left(\frac{d}{r+d}\right)\left(1 / 2-K_B\right)}\log^6(n)~; a=\left\lceil \log ^2(n)\right\rceil.
    $$
    Then, we need to show $\sqrt{\delta } \geq \frac{\tilde{\epsilon}_n \sqrt{n}}{B_n a-\tilde{\epsilon}_n(\log n)(\log \log n)}$. We have, 
    \begin{equation*}
        \begin{split}
            \frac{\tilde{\epsilon}_n \sqrt{n}}{B_n a-\tilde{\epsilon}_n(\log n)(\log \log n)} &= O\left( \frac{\left(B_n\right)^r n^{1/2-\left(\frac{r}{r+d}\right)\left(1 / 2-K_B\right)}}{B_na}  \right) \leq O\left( \frac{\left(B_n\right)^r n^{1/2-\left(\frac{r}{r+d}\right)\left(1 / 2-K_B\right)}}{B_n}  \right) \\
            & = O\left( n^{rK_B + \frac{d}{r+d}(1/2 - K_B)} \right).
        \end{split}
    \end{equation*}
    Obviously, $\sqrt{\delta} \geq \frac{\tilde{\epsilon}_n \sqrt{n}}{B_n a-\tilde{\epsilon}_n(\log n)(\log \log n)}$ when $C_\delta$ is sufficiently large. Thus, we can apply \cref{Lemma:variantofTheorem2}. To get the final rate in \cref{Theorem:esterrorDNN}, notice that 
    \begin{equation*}
        \begin{split}
            \epsilon_n(\delta, a) &:= B_n \sqrt{\frac{a}{n}}\left[\sqrt{Pdim\left(\mathcal{F}_n\right) \log (n)}+\sqrt{\log \log (n)+\delta}\right]+\sqrt{\tilde{\epsilon}_n^2+\mu_n+\theta_n}\\
            & = B_n \sqrt{\frac{a}{n}}\left[\sqrt{n^{2\left(\frac{d}{r+d}\right)\left(1 / 2-K_B\right)}\log^6(n)}+\sqrt{\log \log (n)+C_\delta n^{2rK_B+2\left(\frac{d}{r+d}\right)\left(1 / 2-K_B\right)}\log^6(n)}\right]\\
            & +\sqrt{\tilde{\epsilon}_n^2
             +\mu_n+\theta_n}\\
            & \leq B_n \sqrt{\frac{a}{n}} O( n^{rK_B+\left(\frac{d}{r+d}\right)\left(1 / 2-K_B\right)}\log^3(n)) + \sqrt{\tilde{\epsilon}_n^2+\mu_n+\theta_n} \\
            & = O(n^{rK_B - \frac{r}{r+d}(1/2 - K_B)}\log^4(n)) +\sqrt{\tilde{\epsilon}_n^2+\mu_n+\theta_n} \\
            & \leq O(n^{rK_B - \frac{r}{r+d}(1/2 - K_B)}\log^4(n)) + \sqrt{\mu_n+\theta_n},
        \end{split}
    \end{equation*}
    where the first inequality is trivial and the second inequality is due to $\sqrt{A^2 +B} \leq A + \sqrt{B}$ for $A, B >0$. Then, \cref{Lemma:variantofTheorem2} says that for $n$ sufficiently large, 
$$
\left\|\hat{f}_n-f_0\right\|_{\mathcal{L}^2} \leq C\left(n^{rK_B - \frac{r}{r+d}(1/2 - K_B)}\log^4(n) + \sqrt{\mu_n+\theta_n} \right),
$$
with probability greater than 
\begin{equation*}
    \begin{split}
       & 1-e^{-\delta}-2 \log (n)\left[\frac{n \beta(a)}{a}+2 P\left(\max _{t \in\{1, \ldots, n\}} m_n\left(\boldsymbol{Z}_t\right) \geq 8 B_n\right)\right] \\
        & \geq 1-e^{-\delta}-2 \log (n)\left[\frac{n \beta(a)}{a}+2 P\left(\max _{t \in\{1, \ldots, n\}} Y_t\geq 2B_n\right)\right] \\
        & = 1 - e^{-C_\delta n^{2rK_B+2\left(\frac{d}{r+d}\right)\left(1 / 2-K_B\right)}\log^6(n)} - 2 \log (n)\left[\frac{n C_\beta^{'} e^{-C_\beta \log^2(n)})}{\log^2(n)}+2 P\left(\max _{t \in\{1, \ldots, n\}} Y_t\geq 2B_n\right)\right]  \\
        & = 1 - e^{-C_\delta n^{2rK_B+2\left(\frac{d}{r+d}\right)\left(1 / 2-K_B\right)}\log^6(n)} - 2 \log (n)\left[\frac{C_\beta^{'} n^{1-C_\beta\log n}}{\log^2(n)}+2 P\left(\max _{t \in\{1, \ldots, n\}} Y_t\geq 2B_n\right)\right],
    \end{split}
\end{equation*}
where the first inequality is due to the fact implied in \cref{Eq:maxMY}
\end{proof}

\subsection{Proof of \cref{Prop:consistency}}

\begin{proof}
    From results in \cref{Theorem:esterrorDNN}, we can define the event $J_n :=\{\| \hat{f}_n-f_0 \|_{\mathcal{L}^2}\leq \epsilon_n\}$, s.t., $P(J_n)\to 1$ as $n\to\infty$. Then, we can further define the event $G_n:= J_n \cap \mathcal{E}_n$. Under C1, it is clear than $P(G_n^c)\leq P(J_n^c) + P(\mathcal{E}_n^c)\to 0$. Thus, $P(G_n) \to 1$. 

    On the event $G_n$, we define $g_n : = \hat{f}_n - f_0$ and attempt to show a upper bound for $M_{g_n} := \|g_n\|_{\infty}$. First, under $C3$ and $C1$, $g_n$ is Lipschitz continuous with the Lipschitz constant $\tilde{ \mathcal{L}}_n$ which depends on $\mathcal{L}_n$ and $\|f\|_{C^{k, \alpha}\left([-B_n,B_n]^d\right)}$. Then, we define $r := \frac{M_{g_n}}{2\tilde L_n}$, we have 
    $$|g_n(\bm{y})| \geq\left|g_n\left(\bm{y}^*\right)\right|-\tilde L_n\left\|\bm{y}-\bm{y}^*\right\| \geq M_{g_n}-\tilde L_n r=\frac{M_{g_n}}{2},$$
    for $\left\|\bm{y}-\bm{y}^*\right\| \leq r$, where $\bm{y}^* = \arg\max_{\bm{y}}|g_n(\bm{y})|$. Therefore $g_n(\bm{y})^2 \geq \frac{M_{g_n}^2}{4}$ for $\forall~\bm{y}$, s.t., $\left\|\bm{y}-\bm{y}^*\right\| \leq r$. Define $\rho := \min(r,2B_n\sqrt{d})$. By considering the joint region of a ball centered at $\bm{y}^*$ with radius $r$ and the domain $[-B_n, B_n]^d$, we have 
    $$
    \|g_n\|_{\mathcal{L}^2}^2 \geq \frac{M_{g_n}^2}{4} C_{\pi}d^{-d / 2} \rho^d.
    $$
    If $\rho = r$, we have $\|g_n\|_{\mathcal{L}^2}^2 \geq \frac{C_{\pi}}{2^{d+2} d^{d / 2}} \frac{M_{g_n}^{d+2}}{\tilde L_n^d}$, which implies that $M_{g_n} \leq 2 d^{\frac{d}{2(d+2)}}\left(\frac{\tilde L_n^d}{C_{\pi}}\right)^{\frac{1}{d+2}}\|g_n\|_{\mathcal{L}^2}^{\frac{2}{d+2}}$. If $\rho = 2B_n\sqrt{d}$, we have $M_{g_n} \leq \frac{2\|g_n\|_{\mathcal{L}^2}}{\sqrt{C_{\pi}}(2 B_n)^{d / 2}}$. Thus, we can conclude that 
    \begin{equation}\label{Eq:pointwise1}
        M_{g_n}\leq 2 d^{\frac{d}{2(d+2)}}\left(\frac{\tilde L_n^d}{C_{\pi}}\right)^{\frac{1}{d+2}}\|g_n\|_{\mathcal{L}^2}^{\frac{2}{d+2}} + \frac{2\|g_n\|_{\mathcal{L}^2}}{\sqrt{C_{\pi}}(2 B_n)^{d / 2}} := \Psi_n.
    \end{equation}
Under C3, $\Psi_n$ converges to 0. Thus, we can fix $\forall~\tau>0$, then $\exists~N$ s.t., $\Psi_n\leq \tau$ for all $n> N$. For these $n$, we have $\left\{\left\|\hat{f}_n-f_0\right\|_{\infty}>\tau\right\} \subseteq G_n^c$. Thus, we have 
$$P\left(\left\|\hat{f}_n-f_0\right\|_{\infty}>\tau\right) \leq P\left(G_n^c\right) \longrightarrow 0.$$
Therefore $\left\|\hat{f}_n-f_0\right\|_{\infty} \xrightarrow{p} 0$, which implies our desired results. 

\end{proof}

\subsection{Proof of \cref{Lemma:uniformconsistencyofCDF}}

\begin{proof}
 The proof is mostly the same as the proof of Lemma 2.1 of \cite{wu2024bootstrap}. In their case, a parametric mean model of the autoregressive regression is assumed, so their proof is based on the smoothness condition of the parametric mean model and the technique of \cite{boldin1983estimation}. In particular, one condition on which their proof relies is that the model parameters are consistent with the true parameters. In our case, we have a non-parametric mean model of the autoregressive regression, but the whole function $f_0$ can be estimated by $\hat{f}$ consistently revealed by \cref{Prop:consistency}. 
\end{proof}

\subsection{Proof of \cref{Theorem:validaityofQPI}}

\begin{proof}
  We show the proof of $J=2$ as an example. The proof of higher steps or one-step prediction can be written similarly. The proof is similar to the proof of Theorem 2.3 in \cite{wu2024bootstrap}. To simplify the notation, we take $f : = f_0$. First, by the tower property, we can show that $F_{Y_{n+2}|\bar{Y}_n}(y)$ is equivalent to:
\begin{equation}\label{truedis}
\begin{split}
    F_{Y_{n+2}|\bar{Y}_n}(y) &=  \mathbb{P}(Y_{n+2}\leq y|\bar{Y}_n)\\
      &=  \mathbb{P}(f(\bm{Y}_{n+1}) + \epsilon_{n+2}\leq y|\bar{Y}_n)\\
    & = \mathbb{P}\left(\epsilon_{n+2}\leq y - f(\bm{Y}_{n+1})\bigg\vert \bar{Y}_n\right)\\
    & = \mathbb{E}\left[\mathbb{P}\left(\epsilon_{n+2}\leq y - f([f(\bm{Y}_{n}) + \epsilon_{n+1}, Y_{n}, \ldots, Y_{n+2}])\bigg\vert\epsilon_{n+1},\bar{Y}_n\right)\bigg\vert \bar{Y}_n\right]\\
    & = \mathbb{E}\left[F_{\epsilon}\left(y - f([f(\bm{Y}_{n}) + \epsilon_{n+1}, Y_{n}, \ldots, Y_{n+2}])\right)\bigg\vert \bar{Y}_n\right]\\
    & = \mathbb{E}\left[F_{\epsilon}\left( \mathcal{L}(y,\bar{Y}_n,\epsilon_{n+1})   \right)\bigg\vert \bar{Y}_n\right];
\end{split}
\end{equation}
we use $\mathcal{L}(y,\bar{Y}_n,\epsilon_{n+1})$ to represent $y - f([f(\bm{Y}_{n}) + \epsilon_{n+1}, Y_{n}, \ldots, Y_{n+2}])$ to simplify notations. Similarly, we can analyze $F_{Y^*_{T+2}|\bar{Y}_n}(y)$, it has below equivalent expressions:
\begin{equation}
\begin{split}
    F_{Y^*_{n+2}|\bar{Y}_n}(y) &=  \mathbb{P}(Y^*_{n+2}\leq y|\bar{Y}_n)\\
    & = \mathbb{E}\left[\mathbb{P}\left(\hat{\epsilon}^*_{n+2}\leq \widehat{\mathcal{L}}(y, \bar{Y}_n,\hat{\epsilon}^*_{n+1})\bigg\vert \hat{\epsilon}^*_{n+1},\bar{Y}_n\right)\bigg\vert \bar{Y}_n\right]\\
    & = \mathbb{E}^*\left[\widehat{F}_{\epsilon}\left(\widehat{\mathcal{L}}(y, \bar{Y}_n,\hat{\epsilon}^*_{n+1})\right)\right],
\end{split}
\end{equation}
 where $\widehat{\mathcal{L}}(y, \bar{Y}_n,\hat{\epsilon}^*_{n+1})$ represents $y - \hat{f}([\hat{f}(\bm{Y}_{n}) + \hat{\epsilon}^*_{n+1}, Y_{n},\ldots,Y_{n+2}])$ and $\mathbb{E}^*(\cdot)$ represents the expectation in the bootstrap world, i.e., conditional on $\bar{Y}_n$. Thus, we hope to show:
 \begin{equation}\label{equivatheorem}
     \sup_{|y|\leq B_n}\bigg\vert \mathbb{E}\left[F_{\epsilon}\left( \mathcal{L}(y,\bar{Y}_n,\epsilon_{n+1})   \right)\bigg\vert \bar{Y}_n\right] - \mathbb{E}^*\left[\widehat{F}_{\epsilon}\left(\widehat{\mathcal{L}}(y, \bar{Y}_n,\hat{\epsilon}^*_{n+1})\right)\right] \bigg\vert \overset{p}{\to} 0. 
 \end{equation}
From here, we have $\mathbb{P}\left(\max _{t \in\{1, \ldots, n\}}\left|Y_t\right| \geq B_n\right)=0$ as $n\to\infty$ by A3. In addition, we have a relationship:
 \begin{equation}\label{decompose}
 \begin{split}
     &\mathbb{P}\left( \mathbb{E}^*\left[\widehat{F}_{\epsilon}\left(\widehat{\mathcal{L}}(y, \bar{Y}_n,\hat{\epsilon}^*_{n+1})\right)\right] - \mathbb{E}\left[F_{\epsilon}\left( \mathcal{L}(y,\bar{Y}_n,\epsilon_{n+1})   \right)\bigg\vert \bar{Y}_n\right] \bigg\vert > \epsilon\right) \\
     &\leq\mathbb{P}(\max_{t\in\{1,\ldots,n\}}|Y_{t}|> B_n) \\
     & + \mathbb{P}\left[(\max_{t\in\{1,\ldots,n\}}|Y_{t}|\leq B_n )\bigcap \right. \\
     & \left. \left(\sup_{|y|\leq B_n}\bigg\vert \mathbb{E}^*\left[\widehat{F}_{\epsilon}\left(\widehat{\mathcal{L}}(y, \bar{Y}_n,\hat{\epsilon}^*_{n+1})\right)\right] -\mathbb{E}\left[F_{\epsilon}\left( \mathcal{L}(y,\bar{Y}_n,\epsilon_{n+1})   \right)\bigg\vert \bar{Y}_n\right] \bigg\vert>\epsilon\right) \right].
 \end{split}
 \end{equation}
Thus, to verify \cref{equivatheorem}, we just need to show that the second term of the r.h.s. of \cref{decompose} converges to 0. Also, it is enough for us to analyze the asymptotic probability of the below expression:
\begin{equation}\label{start}
    \sup_{\max(|y|,\max_{t\in\{1,\ldots,n\}}|y_{t}|)\leq B_n}\bigg\vert \mathbb{E}^*\left[\widehat{F}_{\epsilon}(\widehat{\mathcal{L}}(y, \bar{Y}_n,\hat{\epsilon}^*_{n+1})) \right] - \mathbb{E}\left[F_{\epsilon}\left( \mathcal{L}(y,\bar{Y}_n,\epsilon_{n+1})   \right)\right] \bigg\vert>\epsilon.
\end{equation}
Decompose the l.h.s. of \cref{start} as:
\begin{equation}\label{intermediate}
\begin{split}
    &\sup_{\max(|y|,\max_{t\in\{1,\ldots,n\}}|y_{t}|)\leq B_n}\bigg\vert \mathbb{E}^*\left[\widehat{F}_{\epsilon}(\widehat{\mathcal{L}}(y, \bar{Y}_n,\hat{\epsilon}^*_{n+1})) \right] - \mathbb{E}\left[F_{\epsilon}\left( \mathcal{L}(y,\bar{Y}_n,\epsilon_{n+1})   \right)\right] \bigg\vert \leq \\
    & \sup_{\max(|y|,\max_{t\in\{1,\ldots,n\}}|y_{t}|)\leq B_n}\bigg\vert \mathbb{E}^*\left[\widehat{F}_{\epsilon}(\widehat{\mathcal{L}}(y, \bar{Y}_n,\hat{\epsilon}^*_{n+1})) \right]- \mathbb{E}^*\left[F_{\epsilon}(\widehat{\mathcal{L}}(y, \bar{Y}_n,\hat{\epsilon}^*_{n+1})) \right]   \bigg\vert \\
    &+ \sup_{\max(|y|,\max_{t\in\{1,\ldots,n\}}|y_{t}|)\leq B_n}\bigg\vert  \mathbb{E}^*\left[F_{\epsilon}(\widehat{\mathcal{L}}(y, \bar{Y}_n,\hat{\epsilon}^*_{n+1})) \right] -  \mathbb{E}\left[F_{\epsilon}\left( \mathcal{L}(y,\bar{Y}_n,\epsilon_{n+1})   \right)\right]  \bigg\vert.
\end{split}
\end{equation}
Then, we analyze two terms on the r.h.s. of \cref{intermediate} separately. To simplify the notation, we denote $\max(|y|,\max_{t\in\{1,\ldots,n\}}|y_{t}|)\leq B_n$ as $\tilde{\bm{y}} \leq B_n$.  For the first term, we have:
\begin{equation}\label{part1}
\begin{split}
    &\sup_{\tilde{\bm{y}} \leq B_n}\bigg\vert \mathbb{E}^*\left[\widehat{F}_{\epsilon}(\widehat{\mathcal{L}}(y, \bar{Y}_n,\hat{\epsilon}^*_{n+1})) \right]- \mathbb{E}^*\left[F_{\epsilon}(\widehat{\mathcal{L}}(y, \bar{Y}_n,\hat{\epsilon}^*_{n+1})) \right]   \bigg\vert \\
    & \leq \sup_{\tilde{\bm{y}} \leq B_n}\mathbb{E}^*\bigg\vert \widehat{F}_{\epsilon}(\widehat{\mathcal{L}}(y, \bar{Y}_n,\hat{\epsilon}^*_{n+1})) - F_{\epsilon}(\widehat{\mathcal{L}}(y, \bar{Y}_n,\hat{\epsilon}^*_{n+1}))     \bigg\vert\\
    & \leq \sup_{\tilde{\bm{y}} \leq B_n} \bigg\vert \widehat{F}_{\epsilon}(\widehat{\mathcal{L}}(y, \bar{Y}_n,\hat{\epsilon}^*_{n+1})) - F_{\epsilon}(\widehat{\mathcal{L}}(y, \bar{Y}_n,\hat{\epsilon}^*_{n+1}))     \bigg\vert\overset{p}{\to} 0,~\text{under \cref{Lemma:uniformconsistencyofCDF}}.
\end{split}
\end{equation}
For the second term on the r.h.s. of \cref{intermediate}, we have:
\begin{equation}\label{part2}
\begin{split}
    &\sup_{\tilde{\bm{y}} \leq B_n}\bigg\vert  \mathbb{E}^*\left[F_{\epsilon}(\widehat{\mathcal{L}}(y, \bar{Y}_n,\hat{\epsilon}^*_{n+1})) \right] -  \mathbb{E}\left[F_{\epsilon}\left( \mathcal{L}(y,\bar{Y}_n,\epsilon_{n+1})   \right)\right]  \bigg\vert. \\
    &\leq \sup_{\tilde{\bm{y}} \leq B_n}\bigg\vert \frac{1}{n}\sum_{i=1}^{n}F_{\epsilon}(\widehat{\mathcal{L}}(y, \bar{Y}_n,\hat{\epsilon}_{i})) -  \frac{1}{n}\sum_{i=1}^{n}F_{\epsilon}(\mathcal{L}(y, \bar{Y}_n,\epsilon_{i}))   \bigg\vert\\
    &+ \sup_{\tilde{\bm{y}} \leq B_n}\bigg\vert \frac{1}{n}\sum_{i=1}^{n}F_{\epsilon}(\mathcal{L}(y, \bar{Y}_n,\epsilon_{i})) - \mathbb{E}\left[F_{\epsilon}\left( \mathcal{L}(y,\bar{Y}_n,\epsilon_{n+1})   \right)\right]      \bigg\vert. 
\end{split}
\end{equation}
First, we show one complementary result:
\begin{equation}\label{epandhatep}
\begin{split}
    &\mathbb{P}\left(\max_{i=1,\ldots,n}\bigg\vert\epsilon_i - \hat{\epsilon}_i\bigg\vert>\epsilon\right)\\
    & = \mathbb{P}\left(\max_{i=1,\ldots,n}\bigg\vert Y_i - f(\bm{Y}_i) - Y_i + \hat{f}(\bm{Y}_i)          \bigg\vert>\epsilon\right)\\
    & \leq \mathbb{P}\left(\max_{i=1,\ldots,n}|Y_{i-1}|>B_n\right) \\
    &+ \mathbb{P}\left(  \left(\max_{i=1,\ldots,n}|Y_{i-1}|<B_n\right) \bigcap \left(\max_{i=1,\ldots,n}\bigg\vert  \hat{f}(\bm{Y}_i) - f(\bm{Y}_i)\bigg\vert>\epsilon\right) \right)\\
    & \to 0,~\text{under \cref{Prop:consistency}.}
\end{split}
\end{equation}
We further consider two terms on the r.h.s. of \cref{part2} separately. For the first term, by Taylor expansion, we have:
\begin{equation}\label{part2part1}
\begin{split}
    &\sup_{\tilde{\bm{y}} \leq B_n}\bigg\vert \frac{1}{n}\sum_{i=1}^{n}F_{\epsilon}(\widehat{\mathcal{L}}(y, \bar{Y}_n,\hat{\epsilon}_{i})) -  \frac{1}{n}\sum_{i=1}^{n}F_{\epsilon}(\mathcal{L}(y, \bar{Y}_n,\epsilon_{i}))   \bigg\vert\\
    & = \sup_{\tilde{\bm{y}} \leq B_n}\bigg\vert \frac{1}{n}\sum_{i=1}^{n}\left( F_{\epsilon}(\mathcal{L}(y, \bar{Y}_n,\epsilon_{i})) + p_{\epsilon}(o_i)(\widehat{\mathcal{L}}(y,\bar{Y}_n,\hat{\epsilon}_i)\right.\\
    &\left. - \mathcal{L}(y,\bar{Y}_n,\epsilon_i))  \right)  -  \frac{1}{n}\sum_{i=1}^{n}F_{\epsilon}(\mathcal{L}(y,\bar{Y}_n,\epsilon_i)) \bigg\vert\\
    & = \sup_{\tilde{\bm{y}} \leq B_n}\bigg\vert \frac{1}{n}\sum_{i=1}^{n} p_{\epsilon}(o_i)(\widehat{\mathcal{L}}(y,\bar{Y}_n,\hat{\epsilon}_i) - \mathcal{L}(y,\bar{Y}_n,\epsilon_i))     \bigg\vert \\
    & \leq \sup_{\tilde{\bm{y}} \leq B_n} \frac{1}{n}\sum_{i=1}^{n}\bigg\vert  p_{\epsilon}(o_i)(\widehat{\mathcal{L}}(y,\bar{Y}_n,\hat{\epsilon}_i) - \mathcal{L}(y,\bar{Y}_n,\epsilon_i))\bigg\vert  \\ 
    &\leq \sup_{\tilde{\bm{y}} \leq B_n} \sup_{z}|p_{\epsilon}(z)|\cdot\frac{1}{n}\sum_{i=1}^{n}\bigg\vert  \widehat{\mathcal{L}}(y,\bar{Y}_n,\hat{\epsilon}_i) - \mathcal{L}(y,\bar{Y}_n,\epsilon_i)\bigg\vert\\
    & \leq  \sup_{\tilde{\bm{y}} \leq B_n} C\cdot\frac{1}{n}\sum_{i=1}^{n}\bigg\vert  \widehat{\mathcal{L}}(y,\bar{Y}_n,\hat{\epsilon}_i) - \mathcal{L}(y,\bar{Y}_n,\epsilon_i)    \bigg\vert~\text{(under A4)}\\
    & \leq \sup_{\tilde{\bm{y}} \leq B_n,j\in\{1,\ldots,n\}} C\cdot\bigg\vert  \widehat{\mathcal{L}}(y,\bar{Y}_n,\hat{\epsilon}_j) - \mathcal{L}(y,\bar{Y}_n,\epsilon_j) \bigg\vert,
\end{split} 
\end{equation}
where $o_i$ is some point between $\widehat{\mathcal{L}}(y, \bar{Y}_n,\hat{\epsilon}_{i})$ and $\mathcal{L}(y,\bar{Y}_n,\epsilon_i)$. From \cref{Prop:consistency} and the smoothness order of $f_0$ is larger or equal than 1, we have \cref{part2part1} converges to 0 in probability. To see more details, recall that $\widehat{\mathcal{L}}(y,\bar{Y}_n,\hat{\epsilon}_j) - \mathcal{L}(y,\bar{Y}_n,\epsilon_j) = f([f(\bm{Y}_{n}) + \epsilon_{n+1}, Y_{n}, \ldots, Y_{n+2}]) - \hat{f}([\hat{f}(\bm{Y}_{n}) + \hat{\epsilon}_{n+1}, Y_{n},\ldots,Y_{n+2}])  : = f(a) - \hat{f}(\hat{a})$. We can decompose this term as $f(a) - f(\hat{a}) + f(\hat{a})- \hat{f}(\hat{a})$. From here, relying on the Lipschitz continuous property and \cref{epandhatep}, we can show \cref{part2part1} converges to 0 in probability. For the second term on the r.h.s. of \cref{part2}, by the uniform law of large numbers, we have: 
 \begin{equation}
     \sup_{\tilde{\bm{y}} \leq B_n}\bigg\vert \frac{1}{n}\sum_{i=1}^{n}F_{\epsilon}(\mathcal{L}(y,\bar{Y}_n,\epsilon_i)) - \mathbb{E}\left[F_{\epsilon}\left( \mathcal{L}(y,\bar{Y}_n,\epsilon_{n+1})   \right)\right]      \bigg\vert \overset{p}{\to} 0. 
 \end{equation}
 Combine all pieces, \cref{start} converges to 0 in probability, which implies \cref{Eq:validaityofQPI} in \cref{Theorem:validaityofQPI}. To show the last statement \cref{Eq:validaityofQPI2}, we can rely on the Glivenko–Cantelli theorem and triangle inequality. 
\end{proof}

\subsection{Proof of \cref{Lemma:consistenctres}}

\begin{proof}
    To prove the consistency of point predictions with any quantile level, it is just an application of Lemma 1.2.1 of \cite{politis1999subsampling} under the results from \cref{Theorem:validaityofQPI} and A5. Moreover, the QPI will be asymptotically valid since the two endpoints of QPI are just two point predictions corresponding with specific quantile levels.

    To show the consistency of optimal $L_2$ point prediction, we consider
    $$
  \Bigg| \int_{-B_n}^{B_n}  yd\widehat{F}_{Y^*_{n+j}|\bar{Y}_n}(y)-  \int_{-B_n}^{B_n}  ydF_{Y_{n+j}|\bar{Y}_n}(y)\Bigg|.
    $$
    To simply the notation through this proof, we denote $\widehat{F}_{Y^*_{n+j}|\bar{Y}_n}(y)$ and $F_{Y_{n+j}|\bar{Y}_n}(y)$ as $\widehat{F}$ and $F$, respectively. Using integration by parts, we have
 \begin{equation*}
     \begin{split}
         &\Bigg| \int_{-B_n}^{B_n}  yd\widehat{F}_{Y^*_{n+j}|\bar{Y}_n}(y)-  \int_{-B_n}^{B_n}  ydF_{Y_{n+j}|\bar{Y}_n}(y)\Bigg| \\
         & = \Bigg| y\widehat{F}(y)\big|_{-B_n}^{B_n} - \int_{-B_n}^{B_n}\widehat{F}(y)dy -  yF(y)\big|_{-B_n}^{B_n} + \int_{-B_n}^{B_n}F(y)dy \Bigg|\\
         &= \Bigg| B_n(\widehat{F}(B_n) - {F}(B_n)) + B_n({F}(-B_n) -\widehat{F}(-B_n))   + \int_{-B_n}^{B_n}(F(y) - \widehat{F}(y))dy \Bigg| \\
         &\leq 4B_n \sup _{|y| \leq B_n}|\widehat{F}_{Y^*_{n+j}|\bar{Y}_n}(y)- F_{Y_{n+j}|\bar{Y}_n}(y)| = op(1).
     \end{split}
 \end{equation*}
 The last equality is due to A6. Thus, we have the consistency of the optimal $L_2$ point prediction. 
\end{proof}

\subsection{Proof of \cref{Lemma:propertyofbootstrapseries}}

\begin{proof}
    To show the geometric ergodicity of the forward bootstrap series, we just need to verify that the conditions E1 and E2 are still satisfied for $(Y_1,\ldots Y_n)\in \Omega_n$. By the definition of $ \Omega_n$ and $\hat{f}_n \in \mathcal{N}_n$, we can take $C = 2B_n$ and then $\exists \lambda<1$ s.t., E1 is satisfied for $\hat{f}_n$. For E2, we first consider the uncentered residuals, i.e., $\hat\epsilon_t := Y_t - \hat{f}(\bm{Y}_{t-1})$ for $t = 1, \ldots, n$. The application of centered residuals in \cref{alg:PPI} does not influence the asymptotic results by Chebyshev’s inequality and the boundedness of $\hat{f}_n$ and $Y_t$ given  $\Omega_n$. Then, E2 is met when we consider the kernel estimation of the probability density function $\hat{p}_\epsilon$. Lastly, the geometric ergodicity property of the forward bootstrap time series $\{Y^*_t\}_{t=1}^n$ can be proved in a similar manner to show the geometric ergodicity property of the original series; see Theorem 2 of \cite{franke2002properties} for example. 

    To show the stationary distribution of the forward bootstrap series $P^*_{\pi}$ converges to $P_{\pi}$ in probability, we need to verify the six conditions listed in assumption 4 of \cite{franke2002properties} whose variants are listed as follows in our context:
    \begin{itemize}
        \item [(1)] $|
    \hat{f}_n(\bm{y})| \leq \lambda \max \left\{\left|y_1\right|, \ldots,\left|y_d\right|\right\}+C$, for a positive number $\lambda<1$ and a constant C.
        \item [(2)] $\sup _{\bm{y} \in \mathcal{Y}_n}\{|\hat{f}_n(\bm{y})-f(\bm{y})|\}=o_P(1)$ for an appropriate sequence of sets $\mathcal{X}_n \subseteq \mathbb{R}^d$ with $\cdots \subseteq \mathcal{Y}_n \subseteq \mathcal{Y}_{n+1} \subseteq \cdots$ and $P\left(Y_{t} \notin\mathcal{Y}_n\right)=o_P(1)$ as $n\to\infty$.
        \item [(3)]  $\left\|\hat{p}_{\epsilon}-p_{\epsilon}\right\|_{\infty} = o_P(1)$.
        \item[(4)] $\int\left|\hat{p}_{\epsilon}(x)-p_{\epsilon}(x)\right| \mathrm{d} x = o_P(1)$.
        \item[(5)]  $\int|\epsilon| \hat{p}_{\epsilon}(\epsilon) \mathrm{d} \epsilon = O_P(1)$.
        \item[(6)] $\int\left|p_{\epsilon}(z)-p_{\epsilon}(z+c)\right| \mathrm{d} z=O(c)$.
     \end{itemize}
Then, our desired result can be proved in a similar way as the proof of Theorem 3 of \cite{franke2002properties}. The difference is that $P^*_{\pi}$ converges to $P_{\pi}$ with probability tending to 1 in our case. We should mention that the rates of conditions (3), (4), (5) converge to 0 do not matter to prove the consistency of $P^*_{\pi}$. This can be seen from the proof of Theorem 3 in \cite{franke2002properties}, where the so-called decoupling technique is used. In particular, they start both series from a common point $Y_1 = Y_1^* = y_0$ and then show the probability $Y_k \neq Y_k^*$ increasing slowly for $k<n$. By taking an appropriate $k$, they can show that $P^*_{\pi}$ converges to $P_{\pi}$. 

For conditions (1) and (2), they are implied by the aforementioned analysis and \cref{Prop:consistency}, by taking $\mathcal{X}_n = [-B_n,B_n]^d$ and leveraging the finite covering number property of a compact set. For condition (3), we consider 
$$
\|\hat{p}_{\epsilon}(x)-p_{\epsilon}(x)\|_{\infty} \leq \|\hat{p}_{\epsilon}(x)-\tilde{p}_{\epsilon}(x)\|_{\infty}+\|\tilde{p}_{\epsilon}(x)-p_{\epsilon}(x)\|_{\infty},
$$
where $\hat{p}_\epsilon$ is defined in \cref{Eq:gaussianKernel} and  $\tilde{p}_{\epsilon}(x)=\frac{1}{n} \sum_{t=1}^n \frac{1}{h} K\left(\frac{x-\epsilon_t}{h}\right)$; $\epsilon_t$ are true residuals and $K$ is the same Gaussian kernel applied in \cref{Eq:gaussianKernel}. By the standard non-parametric density estimation result, $\|\tilde{p}_{\epsilon}(x)-p_{\epsilon}(x)\|_{\infty} = o_P(1)$, \cite{li2007nonparametric,Silverman1986}. For the term $\|\hat{p}_{\epsilon}(x)-\tilde{p}_{\epsilon}(x)\|_{\infty}$, we consider 
\begin{equation*}
    \begin{split}
        \hat{p}_{\epsilon}(x)-\tilde{p}_{\epsilon}(x)&=\frac{1}{n} \sum_{t=1}^n\left[\frac{1}{h} K\left(\frac{x-\hat{\epsilon}_t}{h}\right)-\frac{1}{h} K\left(\frac{x-\epsilon_t}{h}\right)\right]\\
        &\leq \frac{L}{(n) h^2} \sum_{t=1}^n\left|\hat{\epsilon}_t-\epsilon_t\right|,
    \end{split}
\end{equation*}
where the last inequality comes from the Lipschitz continuity of the Gaussian kernel. Consider $\frac{1}{n} \sum_{t=1}^n\left|\hat{\epsilon}_t-\epsilon_t\right|$, we can define $U_t$ s.t., $\widehat{Q}(U_t) = \hat{\epsilon}_t$, where $\widehat{Q}(\cdot)$ is the quantile function corresponding with $\hat{p}_\epsilon$. Without loss of generality, we can define $\epsilon_t = Q(U_t)$ where $Q$ is the true quantile function for the innovation term. Then, we have $\frac{1}{n} \sum_{t=1}^n\left|\hat{\epsilon}_t-\epsilon_t\right| \leq \sup_u|\widehat{Q}(u) - Q(u)|$. By \cref{Lemma:uniformconsistencyofCDF}, i.e., $\sup _x\left|\widehat{F}_\epsilon(x)-F_\epsilon(x)\right| \xrightarrow{p} 0,$ we can show that $\sup_u|\widehat{Q}(u) - Q(u)|$ is $o_P(1)$ under the further condition of the smoothness of the innovation distribution. With an appropriate $h \to 0$ slowly enough, we have  $\hat{p}_{\epsilon}(x)-\tilde{p}_{\epsilon}(x)\leq \frac{L}{(n) h^2} \sum_{t=1}^n\left|\hat{\epsilon}_t-\epsilon_t\right| = op(1)$ for $\forall~x$. Condition (3) is thus obtained.  

For condition (4), it is easily obtained from condition (3) being satisfied. For condition (5), it is also easy to verify with the kernel density estimator. Condition (6) comes from the assumption A4 easily. Finally, since $\mathbb{P}((Y_1,\ldots Y_n)\in \Omega_n)\to1$ as $n\to\infty$, we get our desired result. 
\end{proof}
\end{document}